\documentclass[lettersize,journal]{IEEEtran}
\usepackage{amsmath,amsfonts}
\usepackage{array}
\usepackage{subfig} %
\usepackage{textcomp}
\usepackage{stfloats}
\usepackage{url}
\usepackage{verbatim}
\usepackage{graphicx}
\usepackage{cite}
\usepackage[pagebackref=true,breaklinks=true,letterpaper=true,colorlinks,bookmarks=false]{hyperref} 

\usepackage{makecell}
\usepackage{boldline}
\usepackage{multirow}   
\usepackage{enumitem}
\usepackage{breqn}
\usepackage{cleveref} %
\usepackage{wrapfig}  
\usepackage{caption} %
\usepackage{booktabs}

\usepackage{makecell}
\usepackage{boldline}
\usepackage{multirow}   
\usepackage{enumitem}
\usepackage{breqn}
\usepackage{cleveref} %
\usepackage{wrapfig}  
\usepackage{caption, subfig} %
\usepackage{booktabs}
\usepackage{amsthm, amssymb}
\usepackage{ragged2e}
\usepackage{bbm}
\usepackage{xr-hyper}
\usepackage{tikz}
\usepackage{pifont}
\usepackage{float}
\usepackage{bm}
\usepackage{mathrsfs}
\usepackage{color, xcolor}
\usepackage[english]{babel}
\usepackage{colortbl}
\usepackage{stackengine}
\definecolor{Gray}{gray}{0.93} 
\definecolor{light-gray}{gray}{0.6}
\definecolor{light-gray-2}{gray}{0.6}
\definecolor{GrayMedium}{gray}{0.85}

\newtheorem{theorem}{Theorem}

\crefname{section}{Sec.}{Secs.}
\Crefname{section}{Section}{Sections}
\Crefname{table}{Table}{Tables}
\crefname{table}{Tab.}{Tabs.}
\crefname{figure}{Fig.}{Figs.} 
\crefname{equation}{Eq.}{Eqs.}
\crefname{thm}{Thm.}{Thm.}
\crefname{algorithm}{Alg.}{Alg.}

\newcommand{\eg}{{\em e.g.}}           %
\newcommand{\ie}{{\em i.e.}}           %

\newtheorem*{theorem-repeat}{Theorem}

\newcommand{\pmhalf}[2]{%
  #1%
  \stackengine{0pt}{\vphantom{#1}}{%
    \shortstack[c]{%
      \scalebox{0.4}[0.4]{$\pm$}\\[-0.5ex]%
      \scalebox{0.5}[0.5]{#2}%
    }%
  }{O}{c}{F}{F}{L}%
}
\definecolor{light-gray}{gray}{0.6}
\definecolor{light-gray-2}{gray}{0.6}
\definecolor{GrayMedium}{gray}{0.85}

\begin{document}

\title{An Adaptor for Triggering Semi-Supervised Learning to Out-of-Box Serve \\ Deep Image Clustering}

\author{
  Yue Duan,
  Lei Qi,
  Yinghuan Shi$^*$,
  Yang Gao
  \thanks{}%
  \thanks{Yue Duan, Yinghuan Shi, and Yang Gao are with the State Key Laboratory for Novel Software Technology, Nanjing University, Nanjing 210023, China (e-mail: yueduan@smail.nju.edu.cn; syh@nju.edu.cn; gaoy@nju.edu.cn).}
  \thanks{Lei Qi is with the School of Computer Science and Engineering, Southeast University, Nanjing 211189, China (e-mail: qilei@seu.edu.cn).}
  \thanks{
    $^*$Corresponding author: Yinghuan Shi.
    \\
     \textbf{Acknowledgments}: This work is supported by NSFC Project (624B2063, 62222604, 62206052, 62536005, 62192783), Jiangsu Frontier Technology R\&D Project (BF2025061), Jiangsu Science and Technology Major Project (BG2024031), Fundamental Research Funds for the Central Universities (020214380120, 020214380128), State Key Laboratory Fund (ZZKT2024A14, ZZKT2025B05), China Postdoctoral Science Foundation (2024M750424), Jiangsu Funding Program for Excellent Postdoctoral Talent (2024ZB242) and Postdoctoral Fellowship Program of CPSF (GZC20240252).
  }
}
 
\markboth{IEEE TRANSACTIONS ON IMAGE PROCESSING}
{Duan \MakeLowercase{\textit{et al.}}: An Adaptor for Triggering Semi-supervised Learning to Out-of-Box Serve Deep Clustering}

\maketitle

\begin{abstract}
  \justifying
  \label{sec_abstract}
  Recently, some works integrate SSL techniques into deep clustering frameworks to enhance image clustering performance. However, they all need pretraining, clustering learning, or a trained clustering model as prerequisites, limiting the flexible and out-of-box application of SSL learners in the image clustering task. This work introduces ASD, an adaptor that enables the cold-start of SSL learners for deep image clustering without any prerequisites. Specifically, we first randomly sample pseudo-labeled data from all unlabeled data, and set an instance-level classifier to learn them with semantically aligned instance-level labels. With the ability of instance-level classification, we track the class transitions of predictions on unlabeled data to extract high-level similarities of instance-level classes, which can be utilized to assign cluster-level labels to pseudo-labeled data. Finally, we use the pseudo-labeled data with assigned cluster-level labels to trigger a general SSL learner trained on the unlabeled data for image clustering. We show the superior performance of ASD across various benchmarks against the latest deep image clustering approaches and very slight accuracy gaps compared to SSL methods using ground-truth, \eg, only 1.33\% on CIFAR-10. Moreover, ASD can also further boost the performance of existing SSL-embedded deep image clustering methods.
\end{abstract}

\section{Introduction}
\label{sec:intro}

\IEEEPARstart{I}{mage} clustering involves organizing images into semantically meaningful groups based on similarity measures, and it is a crucial unsupervised learning technique widely applied in tasks such as action localization \cite{liu2023revisiting}, image retrieval \cite{tesfaye2020constrained}, segmentation \cite{cho2021picie}, and detection \cite{lin2023explore}. Among various unsupervised methods, \textbf{\textit{deep clustering (DC)}}—leveraging the strong representational power of deep neural networks—has become increasingly prominent \cite{xie2016unsupervised,chang2017deep,dang2021nearest,wang2023self}.

Recently, to further improve clustering performance, several DC frameworks have started integrating powerful \textbf{\textit{semi-supervised learning (SSL)}} methods in their later stages \cite{van2020scan,park2021improving,niu2022spice}. SSL techniques have gained popularity due to their capability to achieve strong discriminative feature representations by exploiting large amounts of unlabeled data guided by very limited labeled samples \cite{lee2013pseudo,berthelot2020remixmatch,sohn2020fixmatch,wang2020enaet,duan2022mutexmatch,wang2023freematch}.
However, despite promising integration attempts, current DC methods relying on SSL have inherent limitations: they typically depend heavily on a pretrained clustering head or require explicit preliminary clustering training to generate reliable pseudo-labels. Without this pretraining stage, SSL methods embedded within existing DC frameworks lose their ability to produce stable cluster-level pseudo-labels (as illustrated in Fig. \ref{fig:intro}), significantly limiting the direct applicability and flexibility of these integrated approaches. We discover the difficulty in directly applying SSL for deep image clustering lies in the absence of cluster-level labels to serve as supervision, while traditional SSL requires certain labeled data to provide supervision signals for the learning of unlabeled data. In previous works, SPICE \cite{niu2022spice}, which filters prototypes for reliable cluster-level pseudo-labeling based on predictive confidence, but this all relies on the training of the front-loaded clustering head. In our scenario of cold-starting SSL, measuring predictive confidence is meaningless for a freshly initialized SSL network. For another example, RUC \cite{park2021improving} acquires pseudo-labeled data based on a trained clustering model.

\begin{figure}[t]
  \centering
  \resizebox{\linewidth}{!}{   \includegraphics[]{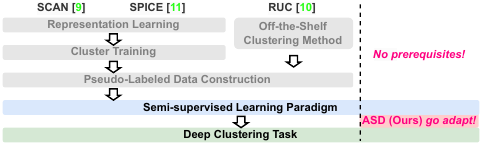}}
  \caption{Different of training framework between SSL-embedded deep image clustering methods and our ASD. SCAN \cite{van2020scan} and SPICE \cite{niu2022spice} utilize representation learning for better feature followed by clustering training, assigning cluster-level pseudo-labels to initiate SSL on unlabeled data. Conversely, RUC \cite{park2021improving} directly incorporates a well-trained deep clustering model to immediately predict cluster-level labels.
   } 
  \label{fig:intro}
\end{figure}

\begin{figure}[t]
  \centering
  \subfloat[]{
  \includegraphics[width=4cm,height=3cm]{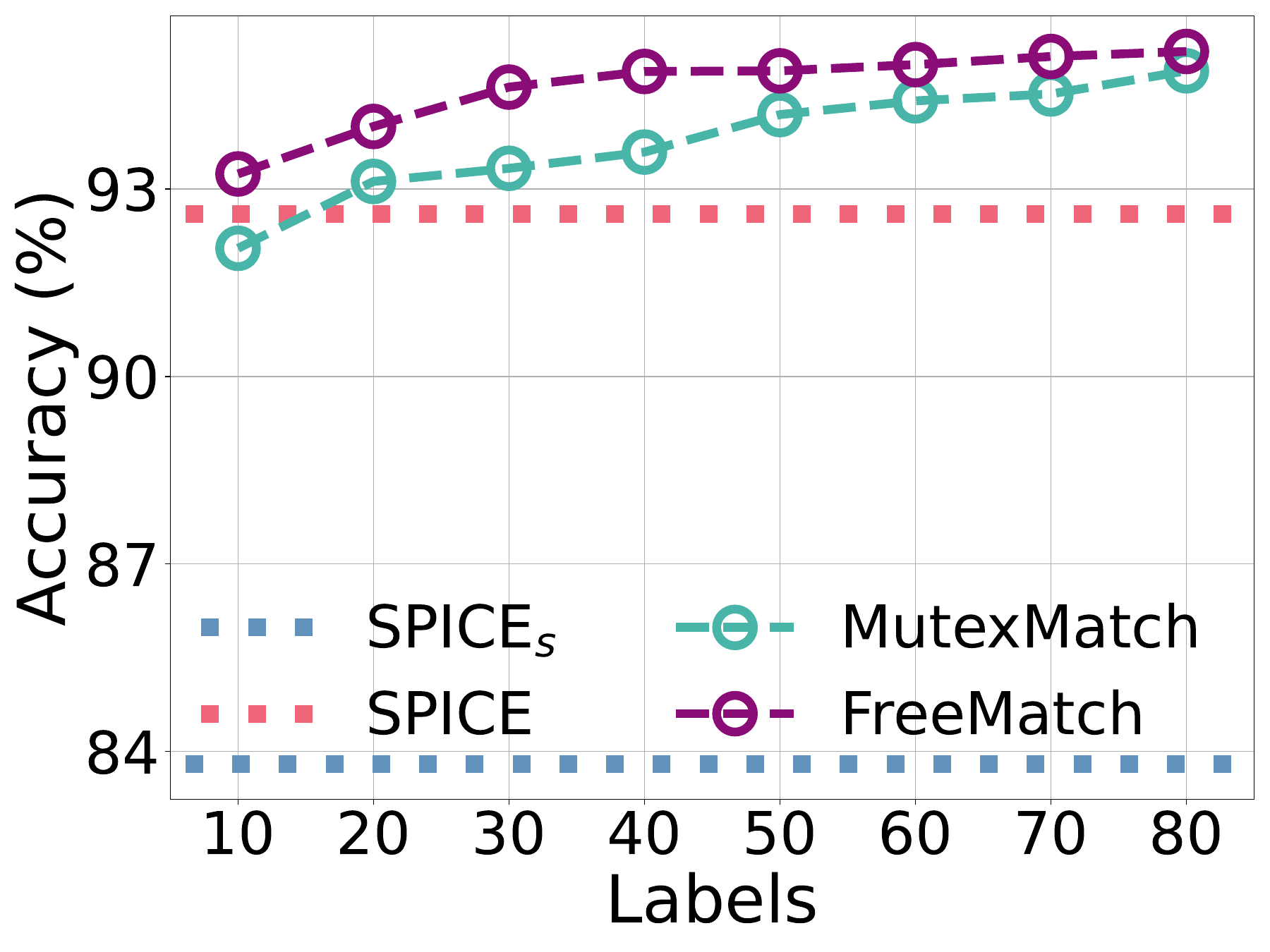}
  \label{fig:intro_exp_1}
  }
\hfil
  \subfloat[]{
  \includegraphics[width=4cm,height=3cm]{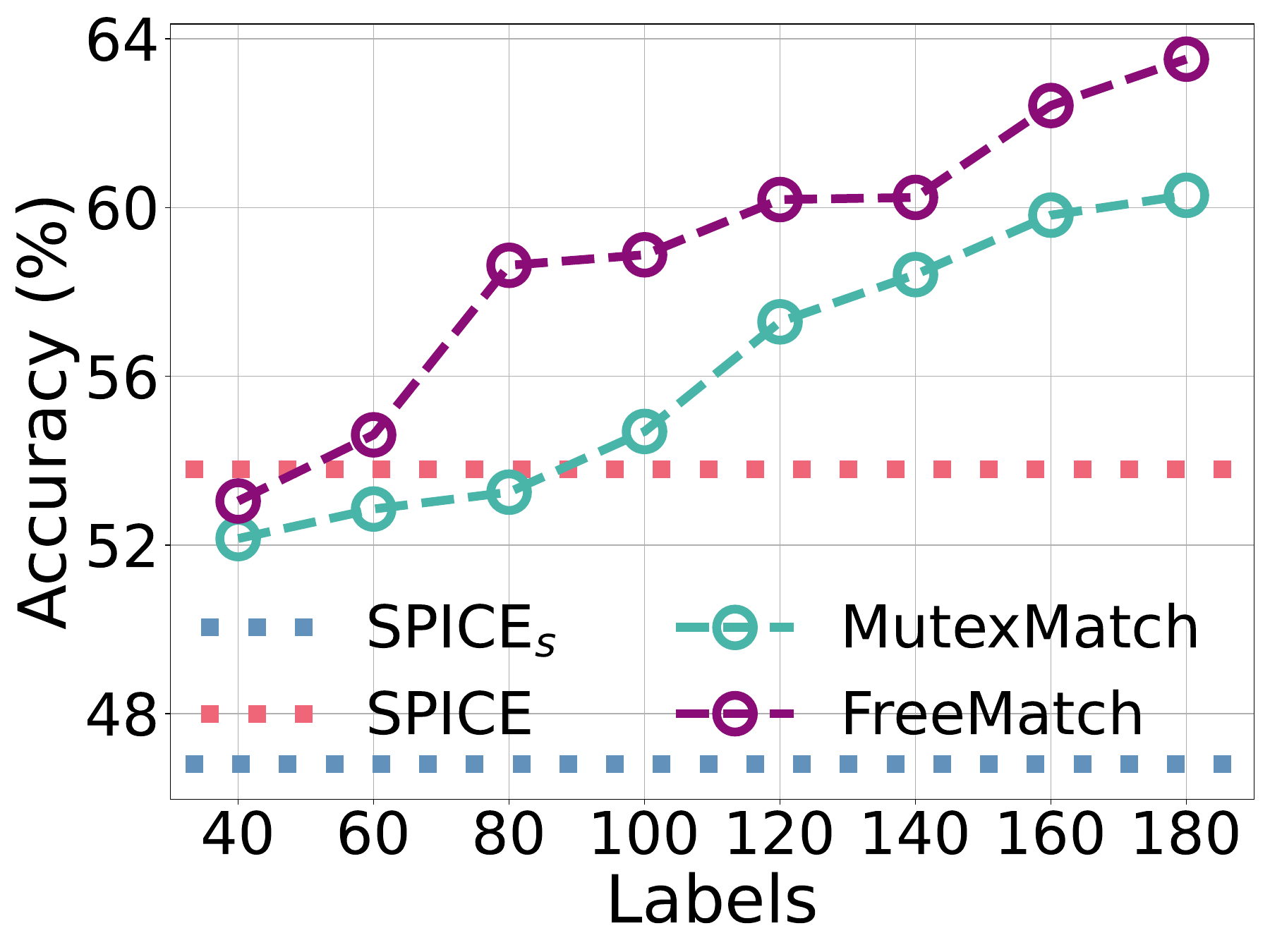}
  \label{fig:intro_exp_2}
  }
  \caption{Experimental observation: advanced SSL method FreeMatch \cite{wang2023freematch} requires only 10 labels to outperform the latest SSL-embedded DC method SPICE \cite{niu2022spice} on CIFAR-10.}
  \label{fig:intro-exp}
\end{figure}

In fact, advanced SSL methods alone often achieve sufficiently satisfactory clustering results without complex DC-specific mechanisms (see empirical evidence in Fig. \ref{fig:intro-exp}). Thus, a natural and intriguing question arises: \textit{Can we directly employ state-of-the-art SSL methods to perform deep image clustering without relying on separate clustering training or labeled data, thereby bridging SSL and DC in a simpler and more efficient manner?} Motivated by this critical insight, we propose \textbf{ASD}, an \textbf{A}daptor for triggering \textbf{S}emi-supervised learning frameworks to perform out-of-the-box \textbf{D}eep image clustering. Specifically, ASD enables the direct utilization of advanced SSL methods in deep clustering tasks without pretraining or explicit clustering head training. To achieve this goal, we identify a fundamental challenge: the \textbf{cold-start problem}, \ie, without pretrained clustering heads or labeled data, directly applying SSL is difficult, as SSL methods inherently rely on initial labeled samples to start effective learning.

To address this, ASD comprises two main components: an instance-level classifier \(G_{ins}\), trained on pseudo-labeled data with instance-level labels, and a cluster-level classifier \(G_{clu}\), trained with stable cluster-level labels mapped from these instance-level labels. Initially, a small subset of pseudo-labeled data is sampled in a manner that ensures representative and comprehensive semantic coverage as much as possible, with each sample assigned a unique instance-level label to provide initial discriminative capability. However, as pseudo-labeled samples are resampled in each iteration, the semantic meanings of instance-level labels may become inconsistent across iterations. To resolve this issue, we utilize optimal transportation to semantically align newly sampled pseudo-labeled data to previously established instance-level classes, maintaining semantic consistency. Next, the crucial challenge is mapping these instance-level pseudo-labels into stable cluster-level labels for clustering supervision. To address this, we propose \textbf{C}lass \textbf{T}ransition \textbf{T}racking (CTT) based label mapping inspired by \cite{duan2023towards}, which clusters instance-level labels at the class-level rather than the sample-level. Briefly, CTT leverages semantic transitions occurring between instance-level classes during SSL training as a similarity measure (see Sec. \ref{sec:ctt} for details), ensuring stable and meaningful semantic grouping for reliable cluster-level supervision. The proposed ASD thus elegantly bridges SSL and DC by addressing the cold-start problem comprehensively, ensuring both semantic consistency and enhanced representativeness, while remaining flexible enough to incorporate future SSL advancements or pretraining strategies seamlessly.

Our main contributions are summarized as follows: \begin{itemize} \item [(1)] We formally propose ASD, a general adaptor framework enabling direct out-of-the-box integration of advanced SSL methods for deep image clustering without requiring pretraining or labeled data. \item [(2)] We introduce a pseudo-labeled data sampling with semantic alignment and CTT-based
label mapping, explicitly addressing the challenges of representativeness and semantic consistency in pseudo-label assignment. \item [(3)] Extensive experiments demonstrate that ASD significantly outperforms or is highly competitive with state-of-the-art DC methods, providing strong empirical validation of our method's effectiveness, robustness, and flexibility in incorporating SSL advancements. \end{itemize}

\section{Related Work}
\label{sec:related}

\subsection{Deep Clustering}
In the realm of deep clustering (DC) research for image clustering task, alternate and simultaneous training techniques have been put forward to boost the clustering performance.  Representative methods like DEC \cite{xie2016unsupervised} first learns an initial feature representation using an autoencoder, and then jointly refines the feature representation and cluster assignments. 
IDEC \cite{guo2017improved} is an extension to DEC, which incorporates a reconstruction loss into the objective function. IDEC further improves the clustering performance by preserving the local structure of the data during the training process. 
JULE \cite{yang2017towards} learns hierarchical cluster assignments and feature representations in an end-to-end manner, leading to improved performance in image clustering tasks. 
In addition, \cite{chang2017deep,chang2019deep,wu2019deep,van2020scan,dang2021nearest} are also typical  approaches, which iteratively refine clustering assignments and optimizing deep neural networks to learn better data representations. It is worth noting that DC has significantly advanced through self-supervised representation learning \cite{bengio2006greedy,ji2019invariant, van2020scan}. For instance, IIC \cite{ji2019invariant} maximizes the mutual information between input images and their cluster assignments. Meanwhile, contrastive learning plays the most important role in simultaneously exploiting discriminative feature representations for DC \cite{li2020prototypical,van2020scan,dang2021nearest,dang2021doubly,niu2022spice,huang2022learning,zhang2024deep}. For examples, SCAN \cite{van2020scan} and NNM \cite{dang2021nearest} pretrain an unsupervised representation learning model with contrastive learning loss, DCDC \cite{dang2021doubly} performs contrastive learning on both sample and class views for more generalizable representation features. 
and SwAV \cite{caron2020unsupervised} uses online clustering to computes cluster assignments for different data augmentations and minimizes the difference between these assignments, yielding strong performance in various downstream tasks. 

\begin{figure*}[t]
  \centering
  \resizebox{0.9\linewidth}{!}{\includegraphics[]{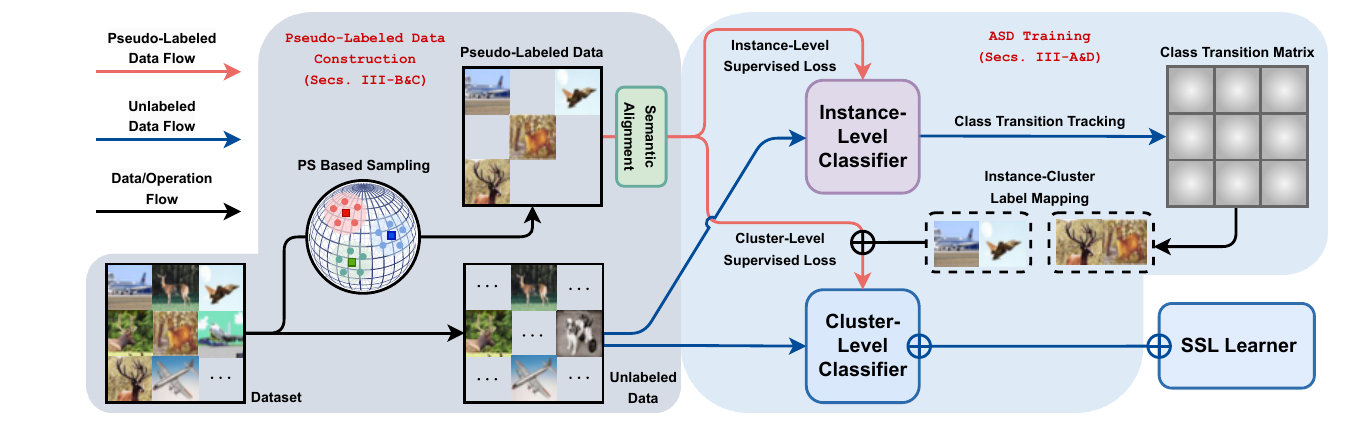}}
  \caption{\textbf{Overview of the proposed ASD.} At the iteration $t$, we sample pseudo-labeled data $x_{l}^{t}$ from the original unlabeled dataset (Secs. \ref{sec:plds} and \ref{sec:proto}). $x_{l}^{t}$ are respectively treated as independent classes and assigned the semantically aligned instance-level labels $y_{l}^{t}$ (see Secs. \ref{sec:ilt} and \ref{sec:sa}). The unsampled data in the dataset are considered as unlabeled data (denoted as $x_{u}^{t}$) in the context of SSL. We use $x_{l}^{t}$ with $y_{l}^{t}$ to train an instance-level classifier and perform class transition tracking on it for $x_{u}^{t}$, enabling us to utilize the information learned from $x_{u}^{t}$ to obtain the cluster-level labels $\phi_{l}(x_{l}^{t})$ of $x_{l}^{t}$ (see Sec. \ref{sec:ctt}). Then, we train a cluster-level classifier with $x_{l}^{t}$ and $\phi_{l}(x_{l}^{t})$ to predict the cluster-level pseudo-labels for $x_{u}^{t}$, so that we can use this classifier to cold-start a generic SSL learner for deep clustering.}
  \label{fig:fw}
\end{figure*}

\subsection{Semi-supervised Learning}

Semi-supervised learning (SSL) is a highly promising paradigm that aims to learn from a large volume of unlabeled data with the help of limited labeled data. A commonly adopted pipeline in SSL involves training a model using labeled data and then using the model’s predictions as pseudo-labels to supervise the unlabeled portion \cite{lee2013pseudo,sohn2020fixmatch}. To enhance the effectiveness of this paradigm, techniques such as consistency regularization \cite{berthelot2019mixmatch,berthelot2020remixmatch,sohn2020fixmatch,duan2022mutexmatch}, contrastive learning \cite{li2021comatch,yang2022class}, and ensemble/mutual learning \cite{li2019semi, wu2019mutual} have been introduced. The former encourages prediction stability under data augmentations, while the latter improves feature representation by aligning positive pairs and separating negatives in the embedding space.

Recent research has focused on improving the quality and reliability of pseudo-labels. For instance, FlatMatch \cite{huang2023flatmatch} uses a cross-sharpness alignment mechanism that penalizes inconsistent predictions, ensuring SSL models do not overfit and remain generalizable. Another line of work detects and mitigates noisy pseudo-labels. DLG \cite{li2023diverse} uses a co-distillation framework to filter noisy labels by cross-verification for robust training. This is conceptually related to DivideMix \cite{li2020dividemix}, which separates clean and noisy samples. NoiseGPT \cite{wang2024noisegpt} detects label noise using probability curvature to measure prediction flatness and correct erroneous labels. Other works like FlexMatch \cite{zhang2021flexmatch} and AdaMatch \cite{berthelot2021adamatch} use adaptive confidence thresholds to refine pseudo-label selection.

Beyond SSL classification, SSL principles have been combined with clustering, for example, in semi-supervised domain adaptation \cite{li2021cross, li2023adaptive}. Works combining DC with SSL for image clustering include SCAN \cite{van2020scan}, RUC \cite{park2021improving}, and SPICE \cite{niu2022spice}. In SCAN, high-confidence samples receive pseudo-labels based on their cluster prediction. RUC considers existing cluster results as a noisy dataset, cleaning samples and then retraining with a refined dataset. SPICE optimizes its network in three stages: training the feature model with contrastive learning, refining cluster semantics with prototype pseudo-labeling, and enhancing performance with reliable pseudo-labeling. In this work, we explore deep image clustering from a new perspective, based on SSL models driven by a self-constructed supervised signal.

\section{Method}
\label{sec:method}
\subsection{ASD}
\label{sec:ssldc}

\noindent\textbf{Preliminary 1: Deep Clustering (DC).}
Given  the unlabeled dataset $\mathcal{D}=\{x^{(1)},\cdots$$,x^{(n)}\}$, we aim to learn a function $f_\theta$ parameterized by $\theta$ (the parameters of the deep neural network) and a set of cluster assignments $\mathcal{C} = \{c^{(1)}, \cdots, c^{(n)}\}$. We assume that $\mathcal{D}$ has $k$ clusters and $k$ is known, \ie, $c^{(i)}\in \mathcal{K}= \{1,\cdots,k\}$. The goal of deep clustering can be reviewed as an optimization task:
\begin{equation}
    \min_{\theta, \mathcal{C}} \sum_{i=1}^{n} \mathcal{L}_{clu}(f_\theta(x^{(i)}), c^{(i)}),
    \label{eq:dc}
\end{equation}
where $\mathcal{L}_{clu}$ is a loss function that encourages similar data points to have the same cluster assignments and dissimilar data points to have different cluster assignments. $\mathcal{L}_{clu}$ could be, for instance, the cross-entropy loss if the cluster assignments are treated as class labels. 

\noindent\textbf{Preliminary 2: Semi-supervised Learning (SSL).} Given the labeled data $x^{(i)}_{l}$ with  corresponding labels $y^{(i)}_{l}$ and the unlabeled data $x^{(i)}_{u}$, we present the loss function of standard self-training-based SSL learner constructed by a feature extractor $F(\cdot)$ and a classifier $G_{clu}(\cdot)$:
\begin{align}\mathcal{L}_{ssl}&=\sum_i\mathcal{L}_{sup}(G_{clu}(F(x^{(i)}_{l})), y^{(i)}_{l})) \nonumber \\
&+\sum_i\mathcal{L}_{unsup}(G_{clu}(F(x^{(i)}_{u})),\phi_{p}(G_{clu}(F(x^{(i)}_{u})))),
\label{eq:ssl}
\end{align}
where $\mathcal{L}_{sup}$ is the supervised loss $\mathcal{L}_{unsup}$ is the unsupervised loss and $\phi_{p}(\cdot)$ is a pseudo-label assignment function for $x^{(i)}_{u}$. 

In order to transform the clustering task into a SSL task, the first step is to provide labeled data to the SSL model. Thus, we first sample $n_{l}$ pseudo-labeled data $x^{(i)}_{l}\in\mathcal{D}_{l}=\{x^{(1)}_{l},\cdots,x^{(n_{l})}_{l}\}$ from $\mathcal{D}$  (Secs. \ref{sec:plds} and \ref{sec:proto}) and assign them cluster-level labels $\phi_{l}(x^{(i)}_{l})$ based on class transition tracking (Sec. \ref{sec:ctt}), where $\phi_{l}(\cdot)$ is the assignment function.  The remaining $n_{u}$ samples  $x^{(i)}_{u}\in\mathcal{D}_{u}= \mathcal{D}\setminus \mathcal{D}_{l}=\{x^{(1)}_{u},\cdots,x^{(n_{u})}_{u}\}$  are regarded as the unlabeled data. Regarding $G_{clu}$ as \textit{cluster-level classifier}, we compute the soft cluster assignment $p^{(i)}=G_{clu}(F(x^{(i)}_{u}))$ for $x^{(i)}_{u}$, where $p^{(i)}_{j}$ can be seen as the probability of sample $x^{(i)}_{u}$ being assigned to cluster $j$ and $j\in \mathcal{K}$. Then, Eq. (\ref{eq:dc}) can be rewritten as 
\begin{equation}
    \min_{F, G_{clu}}\left(\sum_{i=1}^{n_{l}} \mathcal{L}_{clu}(F(x^{(i)}_{l}), \phi_{l}(x^{(i)}_{l})) + \sum_{i=1}^{n_{u}} \mathcal{L}_{clu}(F(x^{(i)}_{u}), p^{(i)})  \right).
    \label{eq:dc2}
\end{equation}
Then, we optimize $F$ and $G_{clu}$ with the help of loss functions utilized in the original SSL learner defined in Eq. (\ref{eq:ssl}), which means we plug $\mathcal{L}_{sup}$ and $\mathcal{L}_{unsup}$ into $\mathcal{L}_{clu}$ that conducted on the pseudo-labeled data and the unlabeled data in Eq. (\ref{eq:dc2}) respectively, \ie, 
\begin{align}
    \min_{F, G_{clu}} &\bigg( \sum_{i=1}^{n_{l}}\mathcal{L}_{sup}(G_{clu}(F(x^{(i)}_{l})), \phi_{l}(x^{(i)}_{l})) \\ \nonumber
    &+\sum_{i=1}^{n_{u}}\mathcal{L}_{unsup}(G_{clu}(F(x^{(i)}_{u})),\phi_{p}(p^{(i)}) \bigg).
\end{align}

As far, we have constructed the \textbf{A}daptor for triggering \textbf{S}emi-supervised learning to out-of-box serve \textbf{D}eep image clustering (ASD) from a high-level perspective. A diagram of ASD is shown Fig. \ref{fig:fw}. For the test phase, we directly use the predictions of $G_{clu}$ to serve as the cluster assignments $\mathcal{C}$. 
Next, we will introduce how to construct pseudo-labeled data, and then how to assign cluster-level labels to them for training.

\subsection{Pseudo-Labeled Data with Instance-Level Label}
\label{sec:ppp}
\subsubsection{Pseudo-Labeled Data Sampling}
\label{sec:plds}
Initially, we introduce a core principle for sampling pseudo-labeled data: \textit{strive to encompass as broad a range of semantic classes as possible within the dataset, ensuring that their semantic span the full class spectrum to the greatest feasible degree}. This concept is grounded in the notion that for an SSL learner to effectively extract clustering knowledge from unlabeled data, it should inherently possess some degree of discriminative capability across all categories. Given the unknown nature of each sample's ground truth, it's impossible to fully ensure adherence to this principle. Nonetheless, our aim is to enhance the likelihood that pseudo-labeled samples represent all semantic categories. Notably, even if the sampled pseudo-labeled data fail to encompass every semantic class, a resilient SSL learner may still develop discriminative abilities for unrepresented classes via exposure to those included (further details in Sec. \ref{sec:ab}). First, we establish a theoretical baseline for the probability that randomly sampled pseudo-labeled labels span all semantic categories.

\begin{theorem}
\label{the}
Given $n$ samples with $k$ classes and a uniform class-distribution (\ie, the number of samples per class is $\frac{n}{k}$), the probability of randomly selecting $n_{l}$ samples ($n_{l} \geq k$) containing all $k$ classes $P_{all}(n_{l}, k, n)$ is given by:
\begin{equation}
P_{all}(n_{l}, k, n) = 1 - \sum_{i=1}^{k} (-1)^{(i-1)} \frac{\binom{k}{i} \binom{n - i(\frac{n}{k})}{n_{l}}}{\binom{n}{n_{l}}},
\end{equation}
where $\binom{a}{b}$ represents the number of combinations.
\end{theorem}

\begin{figure}[t]
  \centering
  \subfloat[]{
  \includegraphics[width=4cm,height=3cm]{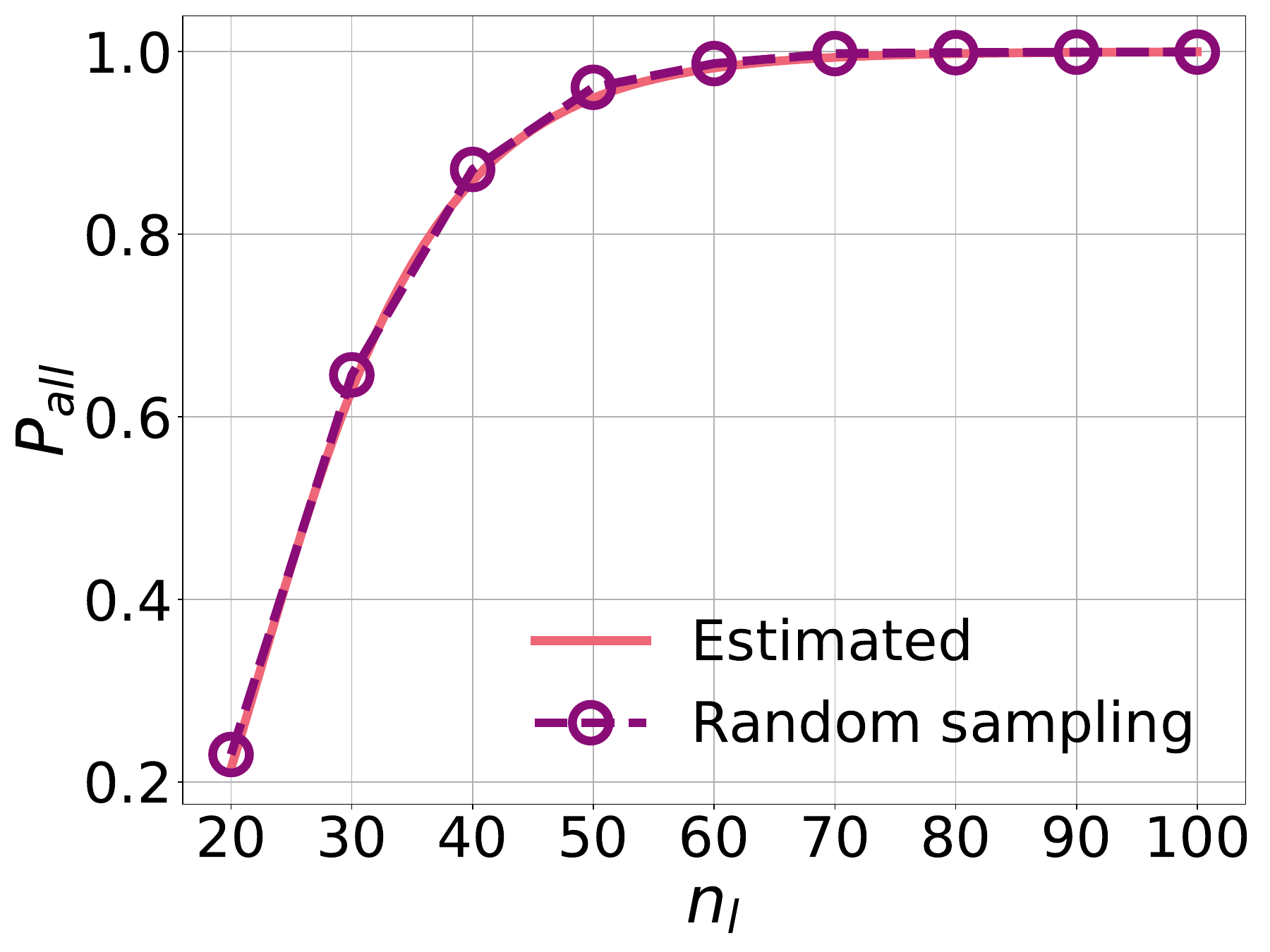}
  \label{fig:sa}
  }
\hfil
  \subfloat[]{
  \includegraphics[width=4cm,height=3cm]{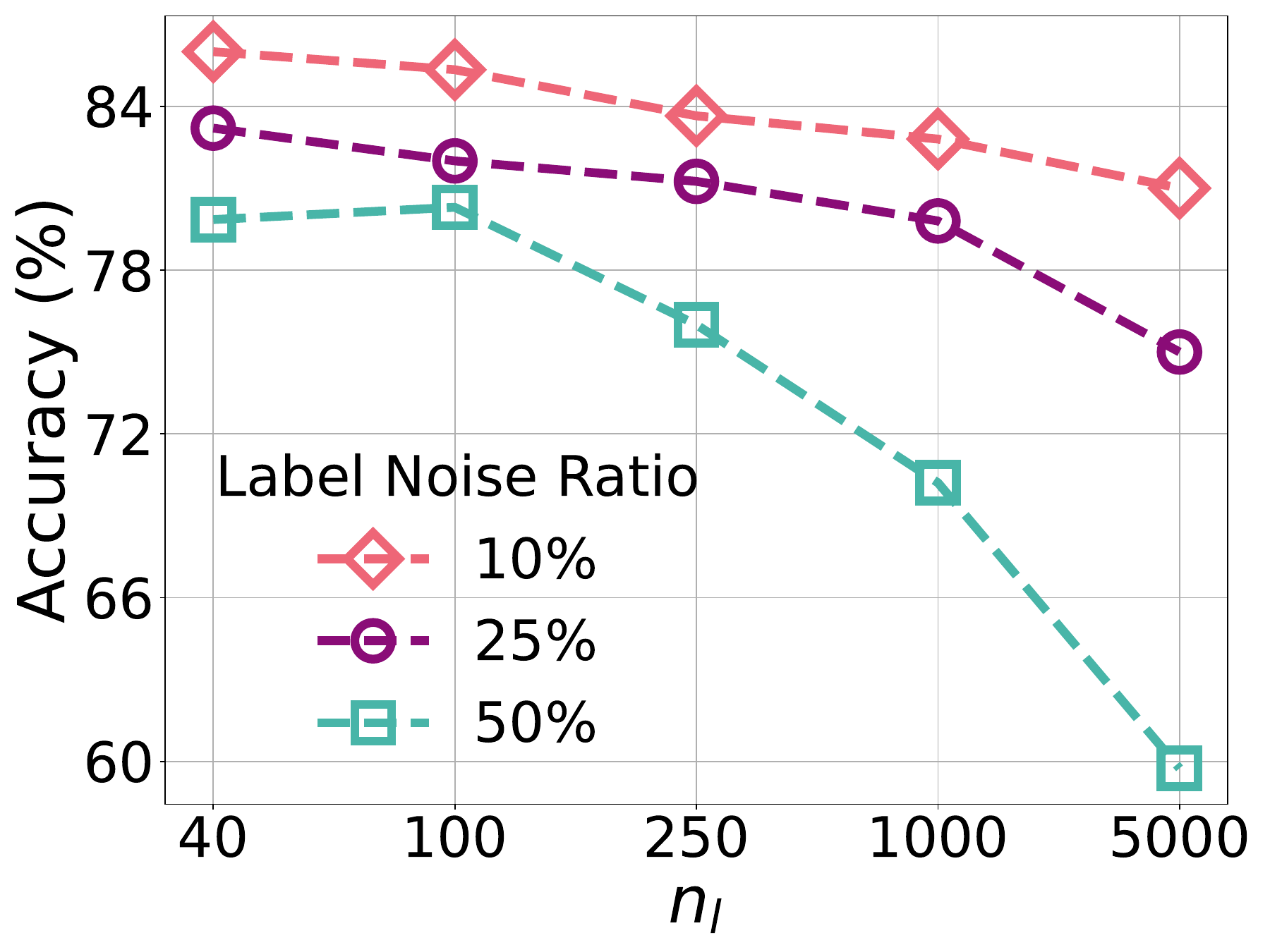}
  \label{fig:lu}
  }
  \caption{(a): $P_{all}$ with various $n_{l}$, fixed $k=10$ and $n=50000$ on CIFAR-10. The results of random  sampling are obtained through frequency statistics on multiple runs. (b): Results of FixMatch \cite{sohn2020fixmatch} (a prevailing SSL learner) on CIFAR-10. The models are trained with different amounts of labels containing fixed ratio of noisy.}
\end{figure}
See Appendix. \ref{app:proof} for detailed proof of Theorem \ref{the}.
For better understanding, we visualize the calculated $P_{all}$ in Fig. \ref{fig:sa} with confirmatory experiments on CIFAR-10. We can observe that as the value of $n_{l}$ increases, $P_{all}(n_{l}, k, n)$ will also increase. 
While increasing $n_{l}$ may be the most straightforward and intuitive method to increase $P_{all}$, blindly increasing $n_{l}$ is not a wise approach. It could lead to an increase in the absolute quantity of noisy labels,  which has negative consequences for the model because lots of noisy labels result in the overfitting of erroneous patterns. No matter how carefully we allocate cluster-level labels to them, it is difficult to completely avoid the presence of noise labels. The absolute increase in the number of noisy labels will deepen the damage to the model, greatly interfering with the process of learning useful information (an example is shown in Fig. \ref{fig:lu}). Thus, we randomly resample the pseudo-labeled data in each iteration with a non-repeating sampler to ensure that all samples are completely accessed in each training epoch, \ie, all semantic classes will be seen in the model training.

\subsubsection{Instance-Level Training}
\label{sec:ilt}
 In the first iteration, we assign a unique instance-level pseudo-label $y^{(i)}_{l}=o_i$ to each $x^{(i)}_{l}$ ($o_i$ denotes the one-hot vector with a 1
 in the $i$-th coordinate and 0's elsewhere), meaning that each $x^{(i)}_{l}$ is treated as an independent class. Then, we set up an instance-level classifier $G_{ins}(\cdot)$ and train it with $x^{(i)}_{l}$ and $y^{(i)}_{l}$. Denoting the cross-entropy loss as $H(\cdot,\cdot)$, the instance-level supervised loss $\mathcal{L}_{ins}$ in the first iteration can be calculated as
\begin{equation}
    \mathcal{L}_{ins} = \sum^{n_{l}}_{i=1}H(G_{ins}(F(x^{(i)}_{l})), o_i).
\end{equation}

\subsubsection{Optimal Transport Based Semantic Alignment}
\label{sec:sa}
Since in each iteration we resample new pseudo-labeled data (denoted as $x_{l}^{(i),t}$), we need to align their instance-level labels to the instance-level semantic classes of pseudo-labeled data from the first iteration (denoted as $x_{l}^{(i),1}$), \ie, ensuring the semantic consistency of the predictions by $G_{ins}$. Although $x_{l}^{(i)}$ in different iterations are not exactly similar at the instance level, they always exhibit certain visual similarities. Thus, we consider combining the information of $\{x_{l}^{(i),1}\}$ to more comprehensively represent $y_{l}^{(i),t}$.

$y_{l}^{(i),t}$ are generated based on the relationship between $x_{l}^{(i),t}$ and $\{x_{l}^{(i),1}\}$, framed as an Optimal Transport (OT) problem \cite{cuturi2013sinkhorn}. Denoting cost function as $\mathbf{O} \in \mathbb{R}^{n_{l}\times n_{l}}$, the cost $O_{ij}=1-F(x_{l}^{(i),t})\cdot F(x_{l}^{(j),1})$ is defined by the negative cosine similarity between the normalized features of $x_{l}^{(i),t}$ and $x_{l}^{(j),1}$. Then, letting $U({\alpha},{\beta})$ be the set of transportation plans $\mathbf{P}$ meeting the flow constraints, with sums of flows from sources to sinks matching vectors $\alpha\in \mathbb{R}^{n_{l}}$ and $\beta\in \mathbb{R}^{n_{l}}$ (both summing to one), we address the following entropic regularized OT problem \cite{cuturi2013sinkhorn}: 
\begin{equation}
    \operatorname*{min}_{\mathbf{P}\,\in U({\alpha},{\beta})}\langle\mathbf{P},{\bf O}\rangle+\lambda \sum_{ij}\mathbf{P}_{ij}\log\mathbf{P}_{ij},
    \label{eq:ot}
\end{equation} 
where $U({\alpha},{\beta})=\{\mathbf{P}\in\mathbb{R}_{+}^{n_{l}\times n_{l}}\mid\mathbf{P}\,{\bf1}_{n_{l}}={\alpha},\;\mathbf{P}^{T}{\bf1}_{n_{l}}={\beta}\}$ and ${\bf1}_{n_{l}}\in \mathbb{R}^{n_{l}}$ is the all-one vector. Considering that we randomly sample from the same dataset, we simply assume that there is no distribution shift between $\{x_{l}^{(i),1}\}$ and $\{x_{l}^{(i),t}\}$. We apply uniform probabilities for $\alpha$ and $\beta$. Following \cite{cuturi2013sinkhorn,tai2021sinkhorn}, we can use an efficient resolution: Sinkhorn-Knopp algorithm \cite{cuturi2013sinkhorn} to solve Eq. \eqref{eq:ot}. With obtained $\mathbf{P}$, we determine soft pseudo-label $y_{l}^{(i),t}\in\mathbb{R}^{n_{l}}$ for $x_{l}^{(i),t}$ by normalizing its row in $\mathbf{P}$ to sum to one, \ie, we calculated the $j$-th element of $y_{l}^{(i),t}$ by $\mathbf{P}_{ij} / \sum_j \mathbf{P}_{ij}$. Then, $\mathcal{L}_{ins}$ is calculated with aligned labels (except for the first iteration):
\begin{equation}
    \mathcal{L}_{ins} = \sum^{n_{l}}_{i=1}H(G_{ins}(F(x^{(i),t}_{l})), y_{l}^{(i),t}).
\end{equation}

\begin{table*}[t!]
  \caption{Clustering performance comparisons on five benchmark datasets. We show the results of ASD run independently and loaded into deep clustering frameworks. \textbf{Bold} and \underline{underline} indicate the best and second best results, respectively. \textcolor{light-gray-2}{Gray results}  use a deeper ResNet-34 backbone (making comparison unfair), whereas results of $^\ast$ are reproduced with ResNet-18 in \cite{zhang2024deep}. For fairness, We cite the results of the original papers of baselines. For SCAN and RUC, since their original papers lack ImageNet-10/-Dogs evaluations, we test them on these datasets based on our re-implementation.}
  \label{tab:cadr}
  \vskip 0in
    \centering
   \footnotesize
    \setlength{\tabcolsep}{1.8mm}{

\begin{tabular}{@{}lccccccccccccccc@{}}
\toprule
Datasets                  &   \multicolumn{3}{c}{CIFAR-10}      &  \multicolumn{3}{c}{ CIFAR-100}   &    \multicolumn{3}{c}{ STL-10}  &  \multicolumn{3}{c}{  ImageNet-10 } &  \multicolumn{3}{c}{  ImageNet-Dogs }    \\ \cmidrule(r){1-1} \cmidrule(lr){2-4} \cmidrule(lr){5-7} \cmidrule(lr){8-10} \cmidrule(lr){11-13}  \cmidrule(lr){14-16}
Metrics                   & ACC  & NMI      & ARI  & ACC  & NMI       & ARI  & ACC  & NMI    & ARI & ACC  & NMI    & ARI  & ACC  & NMI    & ARI \\ \midrule
K-Means \cite{macqueen1965some}          & 22.9 & 8.7      & 4.9  & 13.0 & 8.4       & 2.8  & 19.2 & 12.5   & 6.1& 24.1  &11.9  &5.7 & 10.5 & 5.5 & 2.0  \\
SC \cite{zelnik2004self}              & 24.7 & 10.3     & 8.5  & 13.6 & 9.0       & 2.2  & 15.9 & 9.8    & 4.8 &  27.4 & 15.1 & 7.6 &  11.1 & 3.8 & 1.3 \\
NMF \cite{cai2009locality}               & 19.0 & 8.1      & 3.4  & 11.8 & 7.9       & 2.6  & 18.0 & 9.6    & 4.6 &  23.0 & 13.2 & 6.5 & 11.8 & 4.4 & 1.6  \\
AE \cite{cai2009locality}               & 31.4 & 23.9     & 16.9 & 16.5 & 10.0      & 4.8  & 30.3 & 25.0   & 16.1 &  31.7  &21.0& 15.2 &  18.5 & 10.4 & 7.3 \\
VAE \cite{kingma2013auto}              & 29.1 & 24.5     & 16.7 & 15.2 & 10.8      & 4.0  & 28.2 & 20.0   & 14.6 &   33.4 & 19.3 & 16.8 & 17.9 & 10.7 & 7.9 \\
DCGAN \cite{radford2015unsupervised}           & 31.5 & 26.5     & 17.6 & 15.1 & 12.0      & 4.5  & 29.8 & 21.0   & 13.9 &  34.6 & 22.5& 15.7 & 17.4 & 12.1 & 7.8 \\
SWWAE \cite{zhao2015stacked}            & 28.4 & 23.3     & 16.4 & 14.7 & 10.3      & 3.9  & 27.0   & 19.6   & 13.6 & -  & -    & - & -  & -    & - \\
JULE \cite{yang2017towards}             & 27.2 & 19.2     & 13.8 & 13.7 & 10.3      & 3.3  & 27.7 & 18.2   & 16.4 & 30.0 & 17.5 & 13.8 &  13.8 & 5.4 & 2.8 \\
DEC \cite{xie2016unsupervised}             & 30.1 & 25.7     & 16.1 & 18.5 & 13.6      & 5.0    & 35.9 & 27.6   & 18.6 & 38.1 & 28.2  & 20.3 & 19.5 & 12.2 & 7.9 \\
DAC \cite{chang2017deep}               & 52.2 & 39.6     & 30.6 & 23.8 & 18.5      & 8.8  & 47.0   & 36.6   & 25.7 &    52.7 & 39.4 & 30.2 &  27.5 & 21.9 & 11.1\\
DeepCluster \cite{caron2018deep}       & 37.4 & –        & –    & 18.9 & –         & –    & 33.4 & –      & –    & -  & -    & -  & -  & -    & - \\
ADC \cite{haeusser2019associative}              & 32.5 & –        & –    & 16.0   & –         & –    & 53.0   & –      & –    & -  & -    & -  & -  & -    & - \\
DDC \cite{chang2019deep}               & 52.4 & 42.4     & 32.9 & –    & –         & –    & 48.9 & 37.1   & 26.7& 57.7 &  43.3 & 34.5  & -  & -    & -\\
DCCM \cite{wu2019deep}             & 62.3 & 49.6     & 40.8 & 32.7 & 28.5      & 17.3 & 48.2 & 37.6   & 26.2 &  71.0 & 60.8 & 55.5 & 38.3 & 32.1 & 18.2 \\
IIC \cite{ji2019invariant}       & 61.7 & –        & –    & 25.7 & –         & –    & 61.0   & –      & –   & -  & -    & -  & -  & -    & -  \\
PICA   \cite{huang2020deep}      & 69.6 & 59.1     & 51.2 & 33.7 & 31.0        & 17.1 & 71.3 & 61.1   & 53.1&  87.0 & 80.2 & 76.1 & 35.2 & 35.2 & 20.1 \\ \midrule
\rowcolor{GrayMedium} ASD  (Ours) &&&&&&&&&&&&&&&         \\
\rowcolor{GrayMedium} $+$MutexMatch \cite{duan2022mutexmatch}            & \pmhalf{92.6}{3.0} & 
\pmhalf{75.0}{9.2} & \pmhalf{61.9}{11.4} & \pmhalf{40.2}{3.5} & \pmhalf{38.5}{5.6} & \pmhalf{22.4}{4.4} & \pmhalf{74.2}{4.0} & \pmhalf{62.6}{3.8} & \pmhalf{55.3}{3.5} & \pmhalf{88.3}{1.2} & \pmhalf{83.2}{1.1} & \pmhalf{80.1}{0.9} & \pmhalf{65.1}{2.3} & \pmhalf{61.6}{2.6} & \pmhalf{52.5}{1.4} \\
\rowcolor{GrayMedium} $+$FreeMatch \cite{wang2023freematch}            & \pmhalf{93.1}{1.8} & \pmhalf{79.5}{8.6} & \pmhalf{70.8}{8.4} & \pmhalf{43.2}{4.0} & \pmhalf{43.9}{3.7} & \pmhalf{23.2}{3.2} & \pmhalf{77.1}{1.8} & \pmhalf{70.0}{2.6} & \pmhalf{69.2}{3.3} & \pmhalf{91.1}{1.1} & \pmhalf{84.6}{0.6} & \pmhalf{81.9}{0.5} & \pmhalf{65.9}{2.5} & \pmhalf{62.0}{1.9} & \pmhalf{52.7}{2.7} \\
\rowcolor{GrayMedium} ASD$_{\text{PS}}$  (Ours) &&&&&&&&&&&&&&&         \\
\rowcolor{GrayMedium} $+$MutexMatch\cite{duan2022mutexmatch}       
& \pmhalf{93.1}{1.8} & \pmhalf{85.9}{5.8} & \pmhalf{85.2}{6.5} 
& \pmhalf{43.6}{3.5} & \pmhalf{40.7}{5.0} & {\pmhalf{23.7}{4.0}} 
& \pmhalf{75.9}{4.0} & \pmhalf{62.9}{3.5} & \pmhalf{56.4}{3.9} 
& \pmhalf{89.5}{1.4} & \pmhalf{85.2}{1.0} & \pmhalf{83.7}{0.9} 
& {\pmhalf{66.5}{2.3}} & \pmhalf{62.0}{2.6} & \underline{\pmhalf{52.9}{1.5}} \\
\rowcolor{GrayMedium} $+$FreeMatch\cite{wang2023freematch}            
& \underline{\pmhalf{93.5}{0.6}} & {\pmhalf{86.2}{2.6}} & {\pmhalf{85.9}{2.7}} 
& {\pmhalf{45.2}{2.9}} & {\pmhalf{44.3}{3.2}} & {\pmhalf{23.7}{3.8}} 
& {\pmhalf{78.8}{2.7}} & {\pmhalf{71.5}{4.6}} & {\pmhalf{60.3}{5.5}} 
& {\pmhalf{91.5}{1.2}} & {\pmhalf{86.1}{1.5}} & \pmhalf{82.2}{1.3} 
& \pmhalf{66.3}{2.0} & \underline{\pmhalf{62.7}{2.4}} & \textbf{\pmhalf{53.4}{2.0}} \\\midrule
DCDC  \cite{dang2021doubly}             & 69.9 & 58.5     & 50.6 & 34.9 & 31.0        & 17.9 & 73.4 & {62.1}   & 54.7& -  & -    & - & -  & -    & -  \\
NNM \cite{dang2021nearest}                     & 84.3 & 74.8     & 70.9 & 47.7 & 48.4      & 31.6 & {80.8} & {69.4}   & {65.0} & -  & -    & - & -  & -    & -  \\ 
\textcolor{light-gray-2}{CC \cite{li2021contrastive}}  & \textcolor{light-gray-2}{79.0} &\textcolor{light-gray-2}{70.5} & \textcolor{light-gray-2}{63.7} & \textcolor{light-gray-2}{42.9} &\textcolor{light-gray-2}{43.1} & \textcolor{light-gray-2}{26.6} & \textcolor{light-gray-2}{85.0} &\textcolor{light-gray-2}{76.4}& \textcolor{light-gray-2}{72.6} & \textcolor{light-gray-2}{89.3} &\textcolor{light-gray-2}{85.9}& \textcolor{light-gray-2}{82.2} & \textcolor{light-gray-2}{42.9} &\textcolor{light-gray-2}{44.5}& \textcolor{light-gray-2}{27.4} \\
CC$^\ast$ \cite{zhang2024deep}  &  76.6 & 68.1& 60.6 & 42.6 & 42.4& 26.7 &  74.7 &67.4& 60.6 & 89.5 &86.2 & 82.5&  34.2 &40.1& 22.5 \\
\textcolor{light-gray-2}{ProPos \cite{huang2022learning}}  & \textcolor{light-gray-2}{94.3}  &\textcolor{light-gray-2}{ 88.6} & \textcolor{light-gray-2}{88.4} &\textcolor{light-gray-2}{61.4} &\textcolor{light-gray-2}{60.6}  &\textcolor{light-gray-2}{45.1}& \textcolor{light-gray-2}{ 86.7} &\textcolor{light-gray-2}{75.8} &\textcolor{light-gray-2}{ 73.7}&\textcolor{light-gray-2}{ 95.6} &\textcolor{light-gray-2}{ 89.6}& \textcolor{light-gray-2}{90.6} & \textcolor{light-gray-2}{74.5} & \textcolor{light-gray-2}{69.2}& \textcolor{light-gray-2}{62.7}\\
ProPos$^\ast$ \cite{zhang2024deep}                   &  92.3 & 86.0&84.6& 52.8& 53.8 &36.0 &73.1& 68.7& 61.4&90.0 & 84.8 &81.9& 47.4& 45.9& 33.8 \\ 
HaDis \cite{zhang2024deep} &93.0&\underline{86.9} & \underline{86.2}& \underline{56.3}  &\underline{56.8}& \underline{41.1} & 73.9 &69.6& 62.3& 94.9&88.4& \underline{89.2}&  55.0 &49.6& 37.6
\\
\textcolor{light-gray-2}{DCHL} \cite{huang2024deep} & \textcolor{light-gray-2}{80.1} & \textcolor{light-gray-2}{71.0} & \textcolor{light-gray-2}{65.4} & \textcolor{light-gray-2}{44.6} & \textcolor{light-gray-2}{43.2} & \textcolor{light-gray-2}{27.5} & \textcolor{light-gray-2}{82.1} & \textcolor{light-gray-2}{72.6} & \textcolor{light-gray-2}{68.0} & \textcolor{light-gray-2}{-} & \textcolor{light-gray-2}{-} & \textcolor{light-gray-2}{-} & \textcolor{light-gray-2}{51.1} & \textcolor{light-gray-2}{49.5} & \textcolor{light-gray-2}{35.9} \\
\textcolor{light-gray-2}{IcicleGCN} \cite{xu2024deep} & \textcolor{light-gray-2}{80.7} & \textcolor{light-gray-2}{72.9} & \textcolor{light-gray-2}{66.0} & \textcolor{light-gray-2}{46.1} & \textcolor{light-gray-2}{45.9} & \textcolor{light-gray-2}{31.1} & \textcolor{light-gray-2}{-} & \textcolor{light-gray-2}{-} & \textcolor{light-gray-2}{-} & \textcolor{light-gray-2}{95.5} & \textcolor{light-gray-2}{90.4} & \textcolor{light-gray-2}{90.5} & \textcolor{light-gray-2}{41.5} & \textcolor{light-gray-2}{45.8} & \textcolor{light-gray-2}{27.9} \\
\textcolor{light-gray-2}{MRMCC} \cite{jin2025multi} & \textcolor{light-gray-2}{85.6} & \textcolor{light-gray-2}{76.9} & \textcolor{light-gray-2}{73.1} & \textcolor{light-gray-2}{44.3} & \textcolor{light-gray-2}{47.0} & \textcolor{light-gray-2}{30.4} & \textcolor{light-gray-2}{78.3} & \textcolor{light-gray-2}{69.1} & \textcolor{light-gray-2}{62.6} & \textcolor{light-gray-2}{91.0} & \textcolor{light-gray-2}{85.4} & \textcolor{light-gray-2}{82.7} & \textcolor{light-gray-2}{50.6} & \textcolor{light-gray-2}{50.3} & \textcolor{light-gray-2}{36.2} \\
\midrule
SCAN \cite{van2020scan}     & 81.6 & 71.5     & 66.5 & 44.0   & 44.9      & 28.3 & 79.2 & 67.3   & 61.8 & 89.5 &86.7     & 83.5 &44.6 &45.7 &30.6\\
\rowcolor{GrayMedium} $+$ASD  (Ours)                   
& \pmhalf{84.1}{4.0} & \pmhalf{75.6}{7.2} & \pmhalf{72.7}{8.5} 
& \pmhalf{48.9}{2.6} & \pmhalf{46.2}{5.3} & \pmhalf{30.5}{4.3} 
& \pmhalf{82.4}{2.7} & \pmhalf{69.9}{3.8} & \pmhalf{63.4}{3.5} 
& \pmhalf{92.1}{1.2} & \pmhalf{88.9}{1.2} & \pmhalf{84.7}{1.4} 
& \pmhalf{47.1}{3.9} & \pmhalf{47.2}{4.6} & \pmhalf{31.9}{5.9} \\
RUC \cite{park2021improving}     & 90.1 & -     & - &  54.5   & -      & - & 86.7 & -  &- &91.8 &88.4 &85.8 &56.0 &50.7 &36.5\\
\rowcolor{GrayMedium} $+$ASD (Ours)                   
& \pmhalf{92.0}{2.1} & - & - 
& \pmhalf{55.4}{2.7} & - & - 
& \pmhalf{89.2}{2.8} & - & - 
& \pmhalf{92.4}{1.3} & \pmhalf{89.7}{2.4} & \pmhalf{85.6}{3.6} 
& \pmhalf{57.1}{2.5} & \pmhalf{51.5}{3.7} & \pmhalf{37.0}{4.4} \\ 
SPICE \cite{niu2022spice}    &    92.6 & 86.5& 85.2& 53.8& 56.7& 38.7 & \underline{93.8}& \underline{87.2}& \underline{87.0} & \textbf{95.9} & \textbf{90.2}& \textbf{91.2}& \underline{67.5}& \underline{62.7} &52.6 \\
\rowcolor{GrayMedium} $+$ASD (Ours)                   
& \textbf{\pmhalf{94.2}{0.7}} & \textbf{\pmhalf{87.8}{1.0}} & \textbf{\pmhalf{87.6}{1.4}} 
& \textbf{\pmhalf{56.9}{0.4}} & \textbf{\pmhalf{58.1}{1.5}} & \textbf{\pmhalf{42.2}{2.9}} 
& \textbf{\pmhalf{94.2}{0.3}} & \textbf{\pmhalf{87.5}{0.8}} & \textbf{\pmhalf{87.8}{0.5}} 
& \underline{\pmhalf{95.2}{0.2}} & \underline{\pmhalf{89.8}{0.6}} & \pmhalf{88.4}{0.6} 
& \textbf{\pmhalf{69.1}{1.8}} & \textbf{\pmhalf{63.0}{2.6}} & \textbf{\pmhalf{53.4}{3.0}}  \\ \bottomrule
\end{tabular}
      }
\end{table*}
\begin{table}[t!]
  \footnotesize
  \caption{Comparisons with SPICE \cite{niu2022spice} using ResNet-18 on split datasets, where the training and testing images are mutually exclusive, following the method in \cite{niu2022spice}. SPICEs denotes SPICE without embedded SSL framework.
  }
  \label{tab:appdc}
  \vskip 0in
    \centering
    \setlength{\tabcolsep}{0.7mm}{
\begin{tabular}{@{}lccccccccc@{}}
\toprule
Datasets                  &   \multicolumn{3}{c}{CIFAR-10}      &  \multicolumn{3}{c}{ CIFAR-100}   &    \multicolumn{3}{c}{ STL-10}  \\ \cmidrule(r){1-1} \cmidrule(lr){2-4} \cmidrule(lr){5-7} \cmidrule(lr){8-10} 
Metrics                   & ACC  & NMI      & ARI  & ACC  & NMI       & ARI  & ACC  & NMI    & ARI  \\ \midrule
SPICEs   & 84.5 &73.9 &70.9 &46.8 &45.7 &32.1 & 86.2 & 75.6 & 73.2 \\
 SPICE          &\underline{91.8} &\underline{85.0} &\underline{83.6} &\underline{53.5} &\underline{56.5} &\textbf{40.4} & \underline{92.0} &\underline{85.2} &\underline{83.6}  \\
\rowcolor{GrayMedium} $+$ASD             & \textbf{\pmhalf{92.6}{1.2}} &\textbf{\pmhalf{85.6}{1.8}  }   & \textbf{\pmhalf{84.8}{2.3}} & \textbf{\pmhalf{55.3}{0.9}} & \textbf{\pmhalf{57.7}{2.1}}      & \underline{\pmhalf{40.1}{3.5} }&\textbf{\pmhalf{92.8}{0.6}} &\textbf{\pmhalf{85.4}{2.2} }    &\textbf{\pmhalf{85.2}{2.0} } \\ 
\bottomrule
\end{tabular}}

\end{table}

\subsection{Prototypes Accompanied by Neighbors Based Sampling}
\label{sec:proto}
Considering the fundamental principle for sampling pseudo-labeled data: \textit{including as many semantic classes as possible in the dataset}, the effect of randomly selecting pseudo-labeled data may be unsatisfactory. Thus, we propose the following \textbf{P}rototypes accompanied by neighbors based pseudo-labeled data \textbf{S}ampling (\textbf{PS}). We first apply K-Means algorithm \cite{macqueen1965some}  to cluster all feature vectors extracted by $F$ and thus we can obtain $k$ centroids $\{\mu^{(1)},\cdots,\mu^{(k)}\}$ to served as prototypes. Each $\mu^{(i)}$ represents a group of similar samples, which can be considered as representatives of a specific semantic class. Hereafter, we mine 
$\frac{n_{l}}{k}$ nearest neighbors in $D$ for each $\mu^{(i)}$ over the feature space. We denote the neighboring sample set of $\mu^{(i)}$ as $\mathcal{N}_{\mu^{(i)}}$. Finally, we obtain $\mathcal{D}_{l}=\mathcal{N}_{\mu^{(1)}}\cup\mathcal{N}_{\mu^{(2)}}\cup \cdots\cup \mathcal{N}_{\mu^{(k)}}$.   Since the prototypes $\mu^{(i)}$ are chosen to be diverse and representative, the samples in different $\mathcal{N}_{\mu^{(i)}}$ will likely belong to different semantic classes. This encourages that $\mathcal{D}_{l}$ encompass a broader range of class space, \ie, even if $n_{l}$ is very small, prototype-based sampling can almost always ensure that $D_{l}$ covers all semantic classes ($P_{all}\approx1$).  Although PS incurs additional computational overhead, PS could further boosts ASD, because PS improves the representativeness of pseudo-labeled data. We refer ASD equipped with PS to ASD$_{\text{PS}}$. Moreover, PS benefits more from advanced pretraining technique (discussed in Sec. \ref{sec:dcp}), though this slightly conflicts with our out-of-the-box design philosophy. Nevertheless, our method remains fully compatible with pretraining strategies such as those used in \cite{van2020scan} and \cite{niu2022spice}.

\subsection{Class Transition Tracking Based Label Mapping}
\label{sec:ctt}
Although we've initiated cold-start learning of discriminative features at the instance-level, the current challenge lies in mapping instance-level labels of pseudo-labeled data to cluster-level labels to provide supervision for clustering tasks. A straightforward idea might be to directly cluster pseudo-labeled data using sample-level algorithms like $k$-Means \cite{macqueen1965some}, but this approach is sensible. Since pseudo-labeled data are resampled in each iteration, clustering them at the \textbf{sample-level} would lead to inconsistency in cluster semantics, and cluster matching algorithms (\eg, hungarian matching algorithm \cite{kuhn1955hungarian}) cannot simply rectify this due to the variability in clustering points each time. However, note that since we've aligned the semantics of pseudo-labeled data at the instance-class-level (Sec. \ref{sec:sa}), we can perform clustering on them directly at the \textbf{class-level}, thus using the cluster assignment to decide the label mapping from instance-level to cluster-level. 

In the first iteration, we solely conduct instance-level training, thus eliminating the need for label mapping. In the subsequent iterations, we first compute the instance-level class predictions $\hat{q}^{(i)}=\arg\max(G_{ins}(F(x^{(i)}_{u})))$ for the unlabeled data.  With obtained $\hat{q}^{(i)}$, inspired by \cite{duan2023towards}, we track \textit{class transitions} between consecutive epochs. Denoting $\hat{q}^{(i),e}=a$ as the instance-level class prediction at epoch $e$, during the learning of model, the class transition procedure is defined as the model self-rectify the class prediction to $\hat{q}^{(i),e+1}=b$, where $a\neq b$. 
The model's inconsistent performance on the same sample reflects its difficulty in distinguishing between classes $a$ and $b$, indicating a high degree of similarity between them. Class transitions occurred in $m$-th batch are tracked into $\mathbf{C}^{(m)}\in\mathbb{R}^{n_l\times n_l}_{+}$, where each element $C_{ij}^{(m)}$ is the frequency of class transition and parameterized as follows:
\begin{align}
  C^{(m)}_{ij}=  
  \left| \left\{(i,j)\mid\hat{q}^{(b),e}=i,\hat{q}^{(b),e+1}=j,i\neq j \right\} \right|, 
    \label{eq:track}
\end{align}
where $b\in  \left\{1,...,B\right\}$, $m\in  \left\{1,...,N_{b}\right\}$ and $N_{b}$ is the number of tracked batches with unlabeled data batch size $B$. Finally, the class transition matrix $\mathbf{C}'$ is obtained by averaging on all $\mathbf{C}^{(m)}$, \ie, $C_{ij}'=\sum_{m=1}^{N_{b}}{C^{(m)}_{ij}}/N_{b}$.

\begin{table*}[t]
  \caption{Accuracy (\%) comparisons with SSL methods using $40$ labels, which is same as $n_l$ used in ASD. Results of baselines are referred from \cite{duan2022mutexmatch,wang2023freematch} and we use the same settings as them, \ie, WRN-28-2 and WRN-37-2 \cite{zagoruyko2016wide} are adopted sa backbones for CIFAR-10 and STL-10, respectively.  }
  \label{tab:con}
  \vskip 0in
    \centering
  \footnotesize
     \begin{tabular}{@{}ccccccccc@{}}
\toprule
\multirow{2}{*}{Method}         &             \multicolumn{5}{c}{Semi-supervised Baseline}       & \multicolumn{2}{c}{ASD}            \\ \cmidrule(lr){2-6} \cmidrule(lr){7-8} 
     & MixMatch        & ReMixMatch & FixMatch   & MutexMatch &FreeMatch   & $+$MutexMatch  & $+$FreeMatch              \\ 
     &\cite{berthelot2019mixmatch} & \cite{berthelot2020remixmatch} & \cite{sohn2020fixmatch}  & \cite{duan2022mutexmatch} & \cite{wang2023freematch} & (Ours) & (Ours) \\ \midrule
CIFAR-10    &  63.81$\pm$6.48  & 90.12$\pm$1.03 & 92.53$\pm$0.28  & 94.21$\pm$0.84 & 95.10$\pm$0.04  & 92.88$\pm$0.69 & 93.46$\pm$0.31  \\ 
STL-10    &  45.07$\pm$0.96  & 67.88$\pm$6.24 & 64.03$\pm$4.14  & - & 84.44$\pm$0.55 & 78.32$\pm$1.57 & 79.09$\pm$0.42  \\

\bottomrule
\end{tabular}
\end{table*}
\begin{table*}[t]
  \caption{Accuracy (\%) comparisons with SSL methods using various amounts of ground-truth labels on CIFAR-10. Results of baselines are referred from \cite{duan2022mutexmatch,wang2023freematch}.  }
  \label{tab:appssl}
  \vskip 0in
    \centering
  \footnotesize
  \setlength{\tabcolsep}{2.4mm}{
     ~ \begin{tabular}{@{}ccccccccc@{}}
\toprule
\multirow{3}{*}{Labels}         &             \multicolumn{5}{c}{Semi-supervised Baseline}       & \multicolumn{2}{c}{ASD}            \\ \cmidrule(lr){2-6} \cmidrule(lr){7-8} 
     & MixMatch        & ReMixMatch & FixMatch   & MutexMatch &FreeMatch   & $+$MutexMatch  & $+$FreeMatch              \\ 
     &\cite{berthelot2019mixmatch} & \cite{berthelot2020remixmatch} & \cite{sohn2020fixmatch}  & \cite{duan2022mutexmatch} & \cite{wang2023freematch} & (Ours) & (Ours) \\ \midrule
10     &  34.24$\pm$7.06  & 79.23$\pm$7.48 & 75.21$\pm$7.65  & 66.45$\pm$30.42 & 91.93$\pm$4.24 & \multirow{3}{*}{\rotatebox{90}{92.88~}} & \multirow{3}{*}{\rotatebox{90}{93.46~}}  \\ 
250    &  86.37$\pm$0.59 & 93.70$\pm$0.05 & 95.14$\pm$0.05  & - & 95.12$\pm$0.18 &  &   \\
4000    &  93.34$\pm$0.26  & 95.16$\pm$0.01 & 95.79$\pm$0.08  & 95.63$\pm$0.06 & 95.90$\pm$0.02 & &   \\
\bottomrule
\end{tabular}
  }
\end{table*}
Intuitively, the more similar two classes are, the more likely they are to be misclassified as each other's classes. Thus,  $\mathbf{C}'$ can be regarded as a similarly matrix over instance-level class space, which contains the information learned from total unlabeled data. Then, we define the cluster-level label assignment function $\phi_{l}(x^{(i)}_{l})=\texttt{CluAlg}_{k}(\texttt{Norm}(\mathbf{C}'))_{\hat{y}_{l}^{(i)}}$, where $\hat{y}_{l}^{(i)}=\arg\max(y_{l}^{(i)})$ (\ie, semantically aligning to initial instance-level class), $\texttt{Norm}(\cdot)$ is the normalization operation and $\texttt{CluAlg}_{k}(\cdot)$ is a clustering algorithm that can directly accept a similarity matrix as input to output a cluster assignments set $\{c^{(1)},\cdots,c^{(n_{l})}\}$ with $k$ classes. $\phi_{l}$ is constantly updated with refreshed $\mathbf{C}'$, benefiting from the update of knowledge learned on unlabeled data. However, updating $\phi_{l}$ in each iteration is a waste of computing resources, so we update it every $N_{t}$ iterations. As $\texttt{CluAlg}_{k}$ may assign different cluster indexes to the same group across runs, we employ the Hungarian matching algorithm to align cluster indexes between adjacent runs. Finally, we obtain the total loss:
\begin{align}
    \mathcal{L}&=\mathcal{L}_{ins}+\sum_{i=1}^{n_{u}}\mathcal{L}_{unsup}(G_{clu}(F(x^{(i)}_{u})),\phi_{p}(p^{(i)})) \nonumber  \\
    & +\sum_{i=1}^{n_{l}}\mathcal{L}_{sup}(G_{clu}(F(x^{(i)}_{l})),\texttt{CluAlg}_{k}(\texttt{Norm}(\mathbf{C}'))_{\hat{y}_{l}^{(i)}})  ,
\end{align}
where $\mathcal{L}_{sup}$, $\mathcal{L}_{unsup}$ and $\phi_{p}$ follow the implementation of the adopted SSL learner. See Sec. \ref{sec:imde} for details.

\noindent\textbf{Remark.}
Although ASD primarily aims to cold-start SSL learners for DC, it can also be universally integrated into an existing DC framework to bridge it to an SSL framework. It employs class transition tracking to record the historical information of the model's learning on all unlabeled data, thereby more scientifically providing cluster-level labels for pseudo-labeled data to better unleash the performance potential of SSL in clustering (see Sec. \ref{sec:res} for verification).

\section{Experiments}
\label{sec:exp}
\subsection{Experimental Setup}

\subsubsection{Dataset} ASD is evaluated on five image commonly used in image clustering: CIFAR-10, CIFAR-100 \cite{krizhevsky2009learning}, STL-10 \cite{coates2011an}, ImageNet-10 and ImageNet-Dogs \cite{chang2017deep}. CIFAR-10 and CIFAR-100 comprise 50,000/10,000 training/testing samples, respectively belonging to 10 and 20 classes. STL-10 is extracted from ImageNet \cite{deng2009imagenet}, containing 500 training samples, 800 testing samples from 10 classes and 100,000 out-of-distribution samples. ImageNet-10/-Dogs are also  subsets of ImageNet, consisting of 10/15 classes with 13,000/19,500 samples. Following \cite{huang2022learning},  we employ different image sizes depending on the dataset: 32$\times$32 for CIFAR-10 and CIFAR-100, 96$\times$96 for STL-10, ImageNet-10, and ImageNet-Dogs. 

\subsubsection{Baselines}
Following \cite{van2020scan,niu2022spice,zhang2024deep}, we provide various baseline methods for comparisons, listed in Tab. \ref{tab:cadr} from top to bottom: (1) conventional clustering algorithms: K-Means \cite{macqueen1965some}, SC \cite{zelnik2004self}, NMF \cite{cai2009locality}, (2) DC approaches without using contrastive learning: AE \cite{bengio2006greedy}, VAE \cite{kingma2013auto}, DCGAN \cite{radford2015unsupervised}, SWWAE \cite{zhao2015stacked} JULE \cite{yang2017towards}, DEC \cite{xie2016unsupervised}, DAC \cite{chang2017deep}, DeepCluster \cite{caron2018deep}, ADC \cite{haeusser2019associative}, DDC \cite{chang2019deep}, DCCM \cite{wu2019deep}, IIC \cite{ji2019invariant}) and PICA \cite{huang2020deep}, (3) contrastive learning based DC methods:  DCDC \cite{dang2021doubly}, NNM \cite{dang2021nearest}, CC \cite{li2021contrastive}, ProPos \cite{huang2022learning}, DCHL \cite{huang2024deep}, IcicleGCN \cite{xu2024deep} and MRMCC \cite{jin2025multi}. (4) contrastive learning based DC methods with SSL boosting: SCAN \cite{van2020scan}, RUC \cite{park2021improving} and SPICE \cite{niu2022spice}.

\subsubsection{Evaluation Metrics}
For evaluation metrics, we adopt clustering accuracy (ACC)
\cite{hubert1985comparing}, normalized
mutual information (NMI) \cite{strehl2002cluster} and adjusted rand index (ARI) \cite{li2006relationships} as previous works. Following \cite{van2020scan,dang2021nearest,park2021improving,niu2022spice},  we train our models on the training images and test them on the testing images for CIFAR-10/100-20 and STL-10. For baseline methods, we directly cite the results reported in original papers or related works whereas we report our results averaged on multiple runs.  

\begin{figure}[t]
  \centering
  \subfloat[CIFAR-10]{
  \includegraphics[width=4cm,height=2cm]{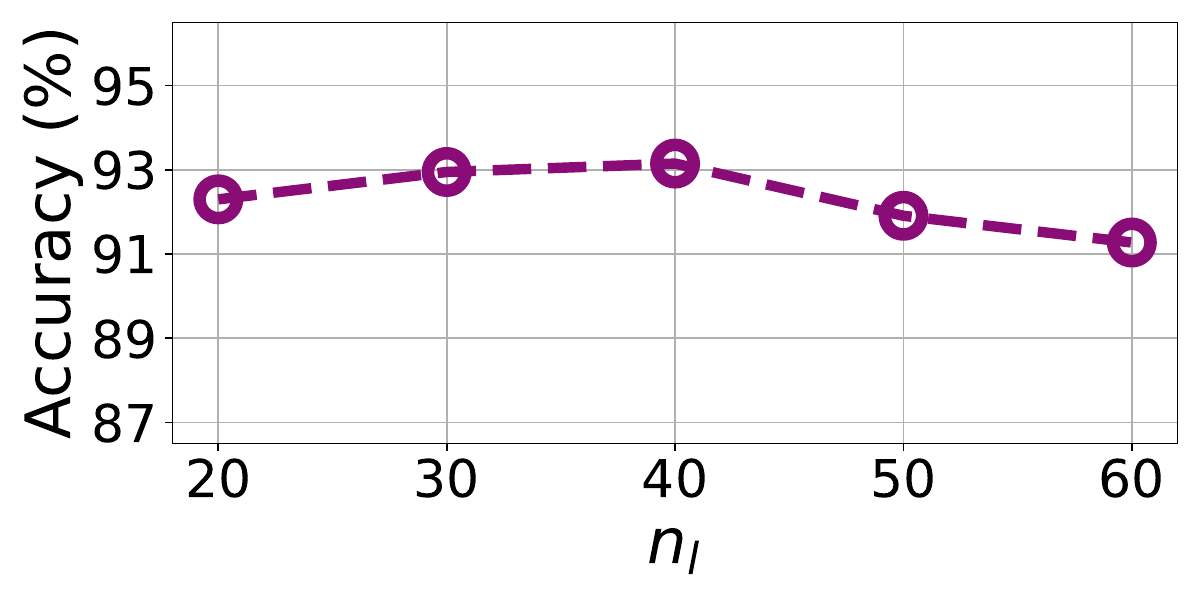}
  \label{fig:exp-a2}
  }
\hfil
  \subfloat[CIFAR-100]{
  \includegraphics[width=4cm,height=2cm]{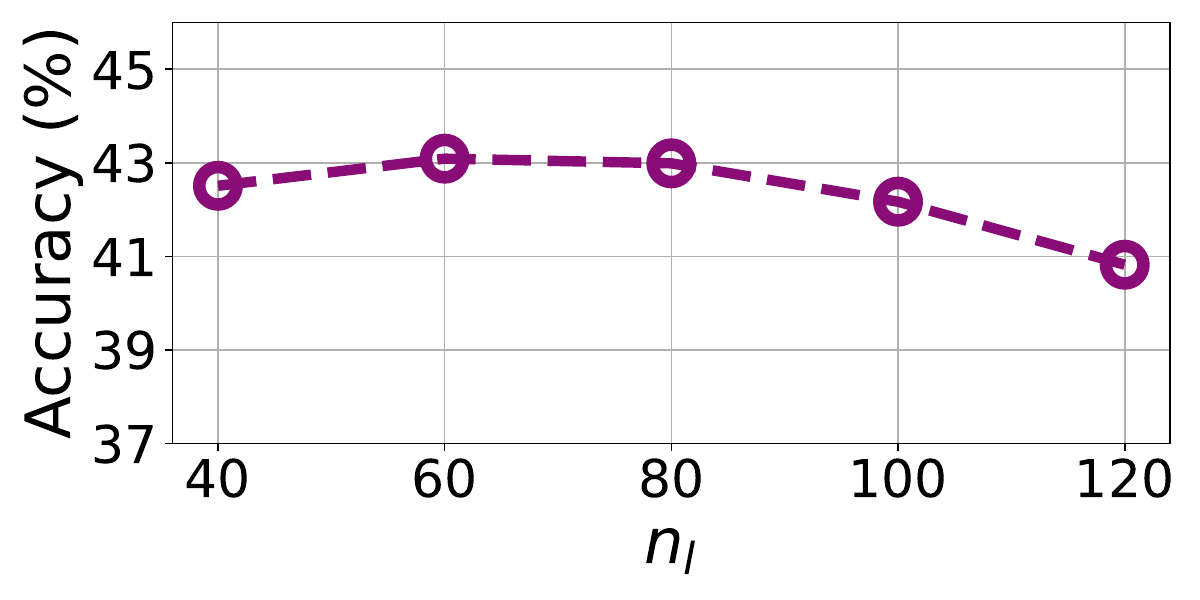}
  \label{fig:exp-b2}
  }
\hfil
  \subfloat[CIFAR-10]{
  \includegraphics[width=4cm,height=2cm]{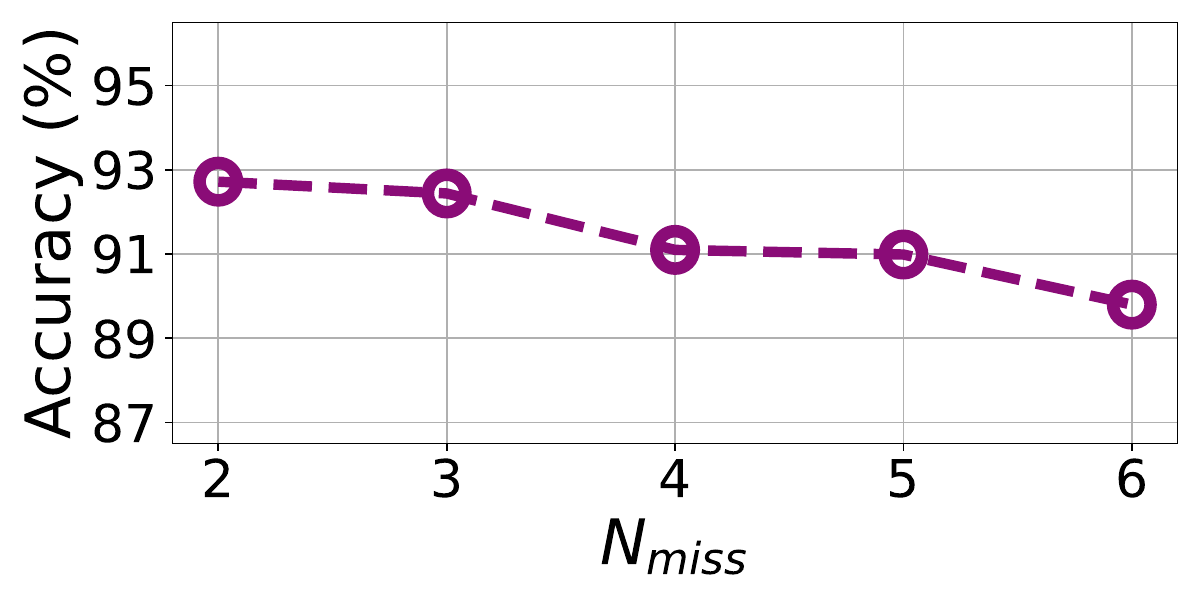}
  \label{fig:exp-c2}
  }
  \hfil
  \subfloat[CIFAR-100]{
  \includegraphics[width=4cm,height=2cm]{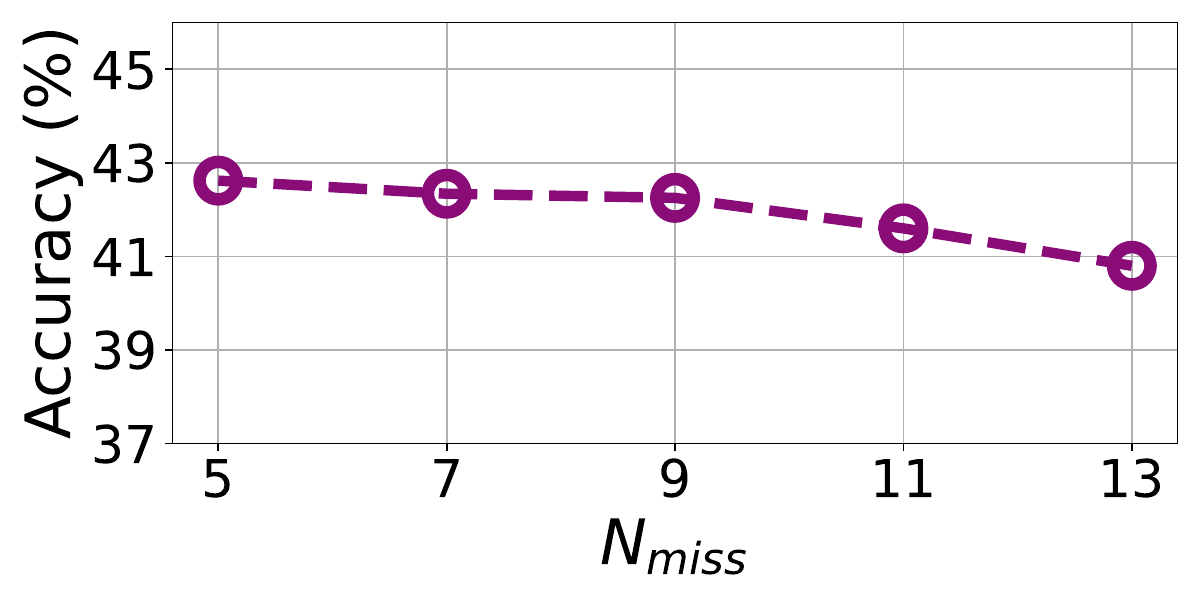}
  \label{fig:exp-d2}
  }
  \caption{Ablations on pseudo-labeled data sampling. (a) and (b): Clustering accuracy with various $n_{l}$ (the default value used in ASD is $4k$). (c) and (d): Clustering accuracy with various numbers of missing classes (denoted as $N_{miss}$) in $D_{l}$.} 
  \label{fig:exp2}
  \vskip 0in
\end{figure}

\begin{figure*}[t]
  \centering
  \subfloat[CIFAR-10]{
  \includegraphics[width=3.7cm,height=3.5cm]{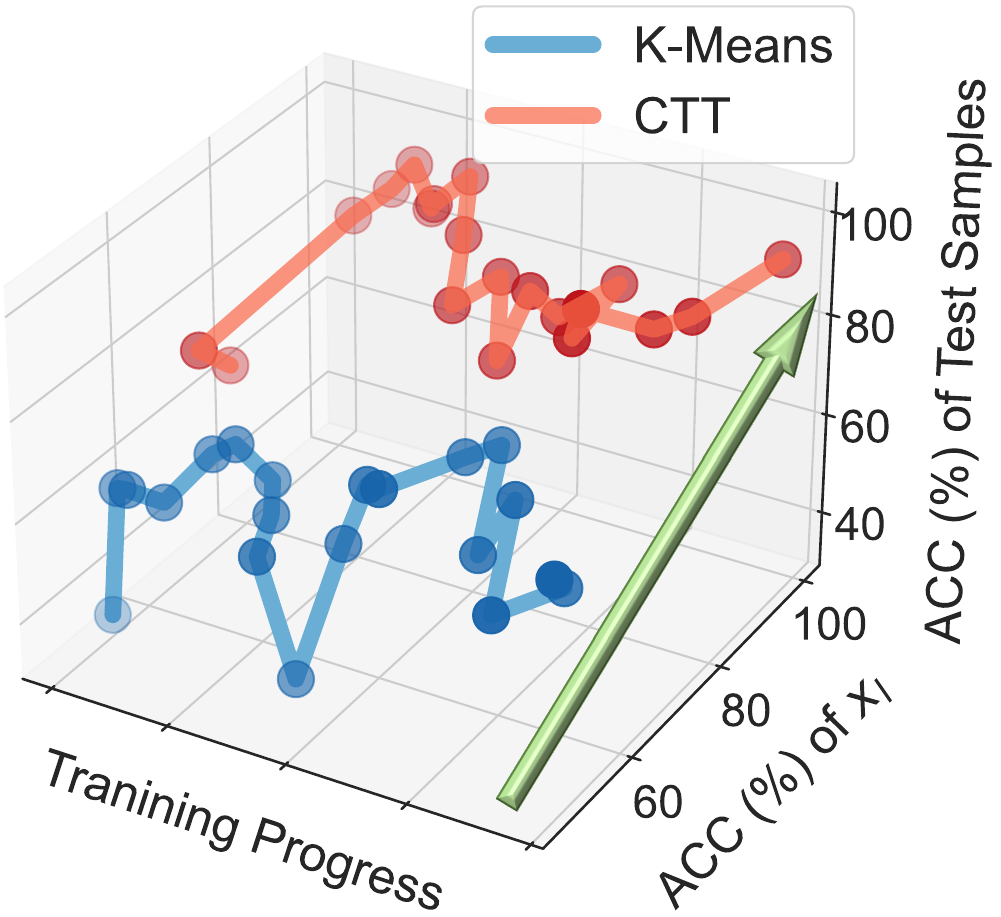}
  \label{fig:exp-a}
  }
\hfil
  \subfloat[CIFAR-100]{
  \includegraphics[width=3.7cm,height=3.5cm]{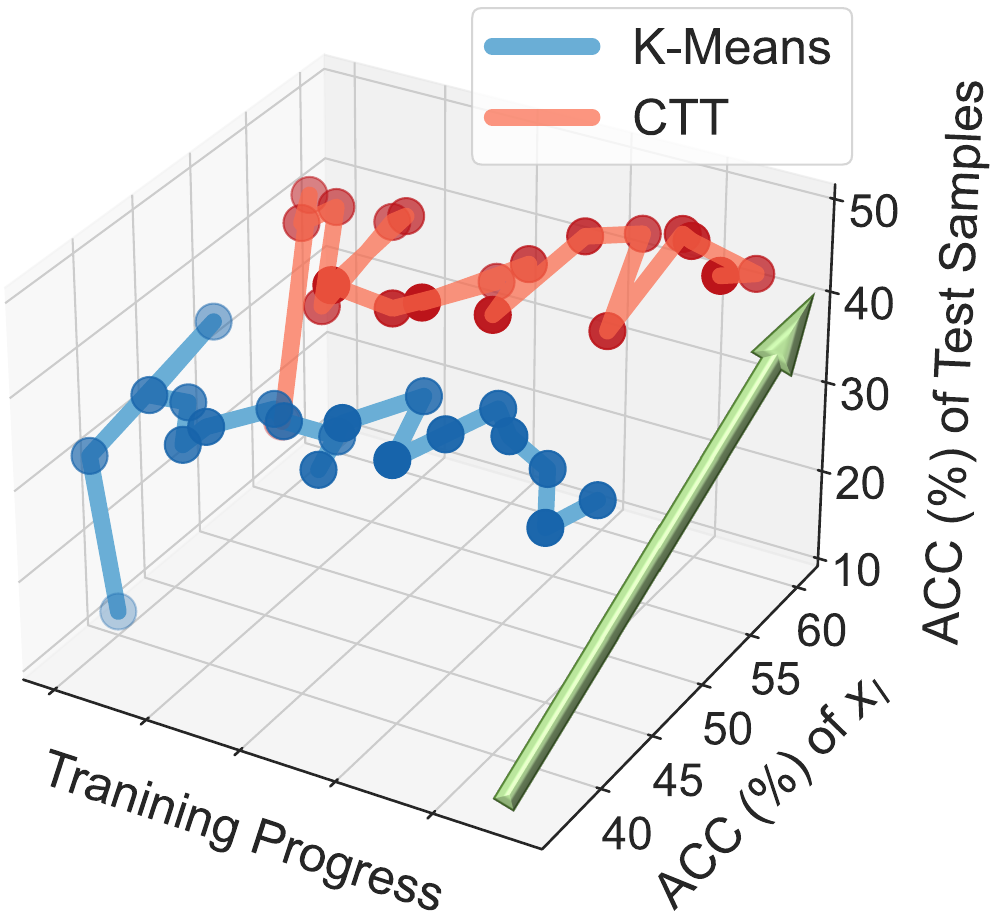}
  \label{fig:exp-b}
  }
\hfil
  \subfloat[$F(x^{(i)}_{l})$]{
  \includegraphics[width=3.5cm,height=3.5cm]{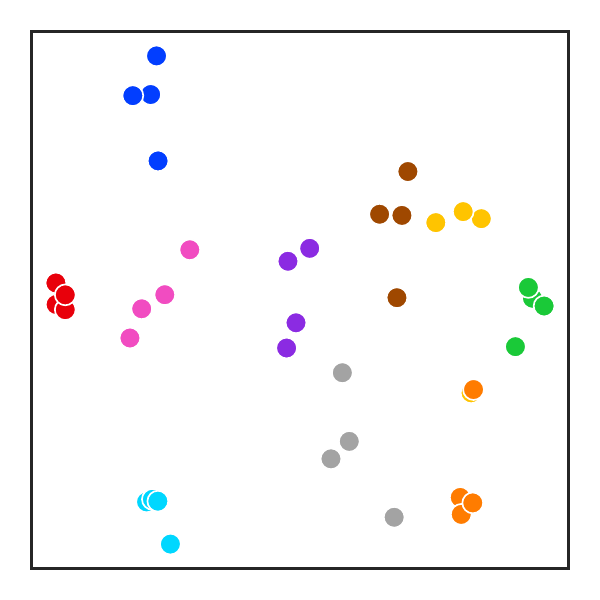}
  \label{fig:exp-c}
  }
  \hfil
  \subfloat[$\texttt{Norm}(\mathbf{C}')$]{
  \includegraphics[width=3.5cm,height=3.5cm]{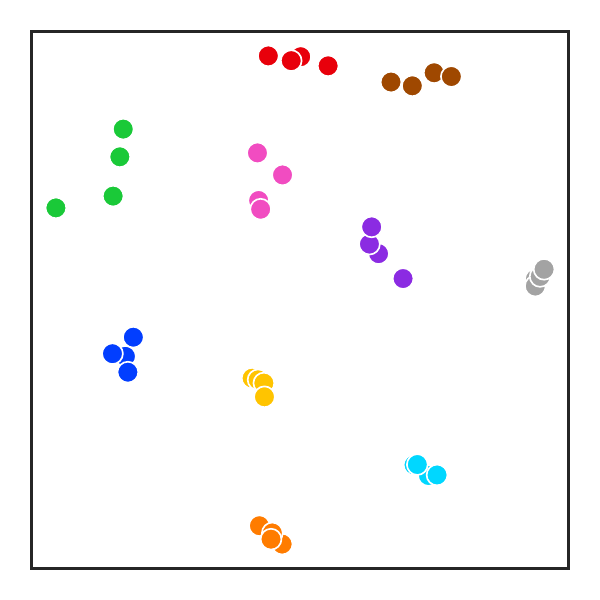}
  \label{fig:exp-d}
  }
  \caption{Experimental comparisons between K-Means and CTT based label assignment for $x^{(i)}_{l}$. (a) and (b): Clustering accuracy of $x^{(i)}_{l}$ and test samples on CIFAR-10/-100. \textcolor[RGB]{130,179,102}{Green arrow} directions indicate better performance. (c) and (d): We visualize $F(x^{(i)}_{l})$ (\ie, features) in (c) and normalized $\mathbf{C}'$ (\ie, class transition matrix) in (d) by t-SNE \cite{van2008visualizing} on CIFAR-10, where different colors represent different ground-truth classes.} 
  \label{fig:exp}
  \vskip 0in
\end{figure*}

\subsubsection{Implementation Details}
\label{sec:imde}
 For DC-framework-free ASD, multiple advanced SSL learners are adopted to comprehensively evaluate the performance of ASD, including MutexMatch \cite{duan2022mutexmatch} and FreeMatch \cite{wang2023freematch}. In the context of MutexMatch, $\mathcal{L}_{sup}$ is the standard cross-entropy loss  (Eq. (2) in \cite{duan2022mutexmatch}), $\mathcal{L}_{unsup}$ is mutex-based consistency loss  (Eqs. (4) and (5) in \cite{duan2022mutexmatch}) and $\phi_{p}$ is the combination of hard and soft pseudo-labeling strategies (Sec. C in \cite{duan2022mutexmatch}). In the context of FreeMatch, $\mathcal{L}_{sup}$ is the standard cross-entropy loss (Eq. (3) in \cite{wang2023freematch}), $\mathcal{L}_{unsup}$ is consistency loss with self-adaptive local and global threshold (Eqs. (8) and (11) in \cite{wang2023freematch}), and $\phi_{p}$ is the combination of hard pseudo-labeling strategies (Sec. 3 in \cite{wang2023freematch}). We follow the original MutexMatch/FreeMatch for the same training setting (\eg, batch size $B=512$) on CIFAR-10/-100 and STL-10. For ImageNet-10/-Dogs, which were not utilized in these two studies, we employ the same training parameters as those used for STL-10. For DC-framework-embedded ASD, our evaluations are primarily conducted on existing DC methods that have already integrated SSL, which can better reflect the superiority of ASD over their original SSL combination strategy. Specifically, we replace their combination strategies with ASD: for SCAN, confidence-based self-labeling (Sec. 2.3 in \cite{van2020scan}) is replaced; for RUC, the mixed sampling strategy (Sec. 3.1 in \cite{park2021improving}) is replaced; and for SPICE, the reliable pseudo-labeling strategy (Sec. III.C in \cite{niu2022spice}) is replaced. We adhere to the original SSL learners used in their methods, \ie, MixMatch \cite{berthelot2019mixmatch} in RUC and FixMatch \cite{sohn2020fixmatch} in SPICE. Specifically, as SCAN does not explicitly use an existing SSL learner, we incorporate FixMatch, which shares similarities with its self-learning strategy. For the training settings, we follow their original literature. Specifically, in subsequent experiments, unless otherwise stated, we use ``ASD+X'' to indicate that the ASD framework is applied on top of a base SSL learner X. Conversely, ``Y+ASD'' denotes that the component connecting the deep clustering and SSL parts in an SSL-embedded DC framework Y is replaced with our proposed ASD module.
 
 Following \cite{van2020scan,dang2021nearest,park2021improving}, we mainly adopt ResNet-18 \cite{he2016deep} for backbone.
 For the number of pseudo-labeled data $n_{l}$, we empirically set $n_{l}=4k$, \ie, $40,80,40,40$ and $60$ for CIFAR-10, CIFAR-100, STL-10, ImageNet-10 and ImageNet-Dogs, respectively. For other algorithm dependent hyper-parameters, we set $N_{b}=1000$ (the number of tracked batches) and $N_{t}=1000$ (the update frequency of $\phi_{l}$) for all datasets. We adopt $k$-Medoids clustering algorithm \cite{kaufman2009finding} and min-max normalization for $\texttt{CluAlg}_{k}(\cdot)$ and $\texttt{Norm}(\cdot)$, respectively. Our models are implemented by PyTorch \cite{paszke2019pytorch} and trained on 6 GeForce RTX 3090 GPUs. The efficiency analysis can be found in Sec. \ref{sec:time}.

\noindent\textbf{Remark.} 
Since baselines methods adopt inconsistent backbone networks for evaluation, to ensure fairness, we maintain ResNet-18 for our strongest competitors, including NNM \cite{dang2021nearest}, CC \cite{li2021contrastive}, ProPos \cite{huang2022learning}, HaDis \cite{zhang2024deep}, SCAN \cite{van2020scan}, RUC \cite{park2021improving}, and SPICE \cite{niu2022spice}. However, the original CC and ProPos use ResNet-34; thus, we refer to the results reproduced by HaDis using ResNet-18. Specifically, SPICE evaluates using both ResNet-18 (on partial datasets we used) and ResNet-34. For fairness, we assess SPICE-embedded ASD with ResNet-34, whereas the comparisons using ResNet-18 appear in Tab. \ref{tab:appdc}.

\subsection{Main Results}
\label{sec:res}
\subsubsection{Clustering Performance Comparison}
The main comparison results for clustering are summarized in Tab. \ref{tab:cadr}. 
We emphasize that the core purpose of ASD is to {cold-start SSL learners for DC} without containing any modules specifically designed for clustering. However, it's noteworthy that the application of \textit{contrastive learning} significantly enhances the performance of DC frameworks \cite{zhou2022comprehensive,ren2022deep}, such as SPICE \cite{niu2022spice}, ProPos \cite{huang2022learning}, HaDis \cite{zhang2024deep}. Therefore, we separate the baseline methods into those using and not using contrastive learning for a fairer comparison. Without contrastive learning, independent ASD consistently achieves higher performance than baseline methods across all datasets, benefiting from the strong learning ability of the SSL model under the effective supervision provided by our ASD. Additionally, FreeMatch-embedded ASD shows more superior performance than MutexMatch-embedded ASD, demonstrating that future, more advanced SSL methods will continue to enhance ASD.

With contrastive learning, DC-framework-embedded ASD still outperforms the various baseline methods, \eg, our ASD improves ACC, NMI, and ARI by 20.3\%, 12.6\%, and 14.4\% respectively over the most recently results reported by HaDis \cite{zhang2024deep} on STL-10. Moreover, compared to SSL-embedded DC methods, using ASD to bridge SSL learners consistently enhances performance in most scenarios. These results proves the efficacy of ASD surpasses previous SSL connection strategy pseudo-labeled data generation with class transition tracking. ASD logs the learning from unlabeled data and scientifically assigns cluster-level labels to pseudo-labeled data, thereby enhancing the clustering potential of SSL.  On ImageNet-10, while ASD does not exceed the originally reported SPICE results, it consistently improves over our reproduced SPICE baseline (\eg, ACC: 94.7$\rightarrow$95.2). Since SPICE have achieved very strong performance, slight fluctuations due to implementation or environment differences are expected.

We observe ASD$_{\text{PS}}$ that adopting PS leads to further boost the performance of ASD. Random sampling may lead to suboptimal pseudo-labels with limited class coverage, potentially affecting the quality of early-stage supervision and causing performance fluctuations. By contrast, PS enhances the representativeness of the sampled data resulting more robust performance. For more discussion, please refer to Sec.~\ref{sec:dcp}.

In the original SPICE \cite{niu2022spice}, the backbone  ResNet-18 (the same one mainly used by SCAN \cite{van2020scan}, NNM \cite{dang2021nearest}, RUC \cite{park2021improving} and our ASD) is used for performance evaluation, but experiments were only conducted on CIFAR-10, CIFAR-100 and STL-10. While SPICE performs complete experiments using ResNet-34 on all the datasets used in Tab. \ref{tab:cadr}, therefore, we use ResNet-34 to compare with it in the main text. In order to demonstrate the advantages of our method more comprehensively, we additionally implement ASD based on ResNet-18 and the results are shown in Tab. \ref{tab:appdc}. We can observe that ASD still achieves considerable performance improvements, demonstrating its robustness to backbones.
\begin{figure*}[t]
  \centering
  \subfloat[CIFAR-10]{
  \includegraphics[width=4cm,height=3cm]{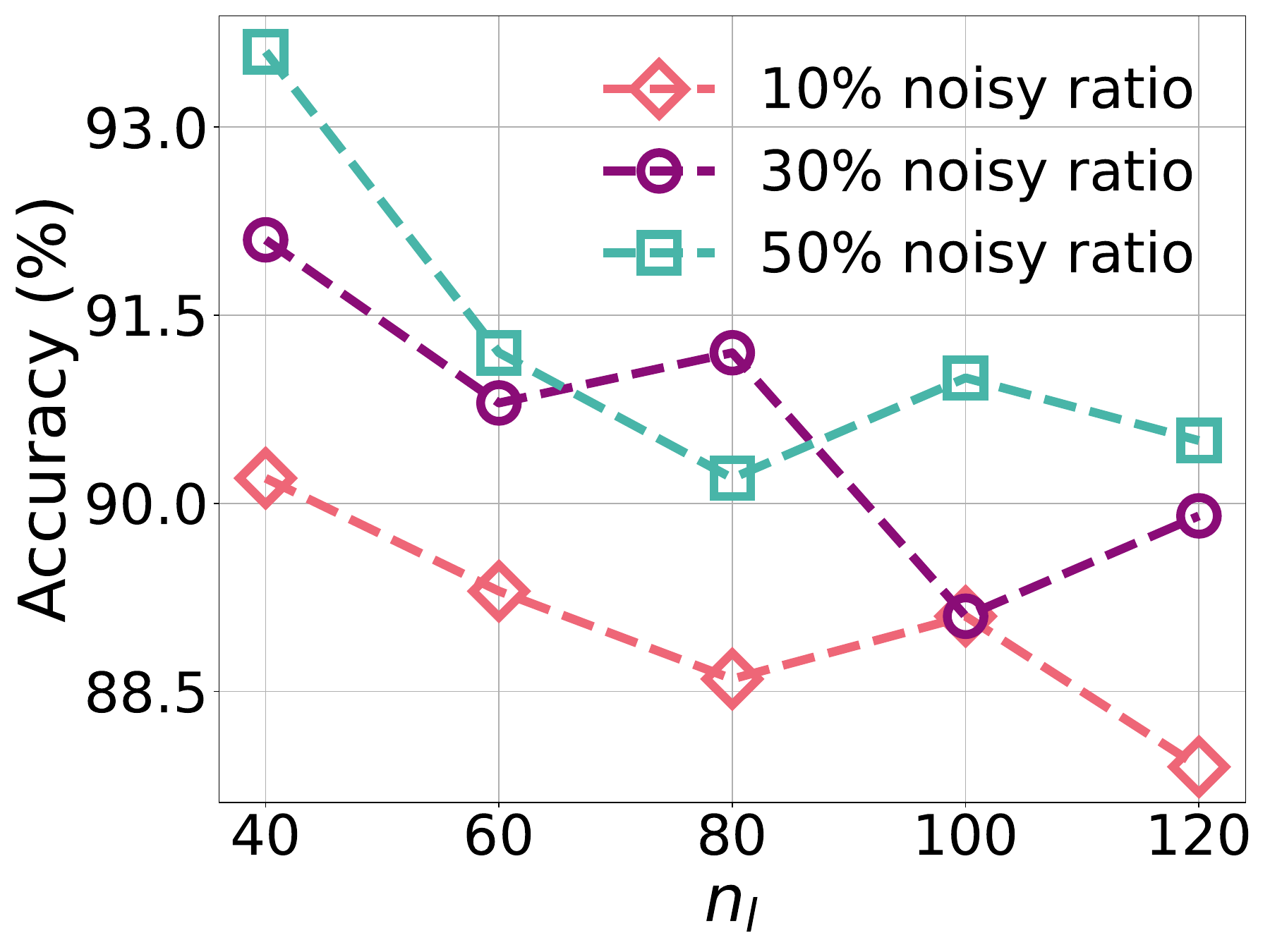}
  \label{fig:noisy_r-1}
  }
\hfil
  \subfloat[CIFAR-100]{
  \includegraphics[width=4cm,height=3cm]{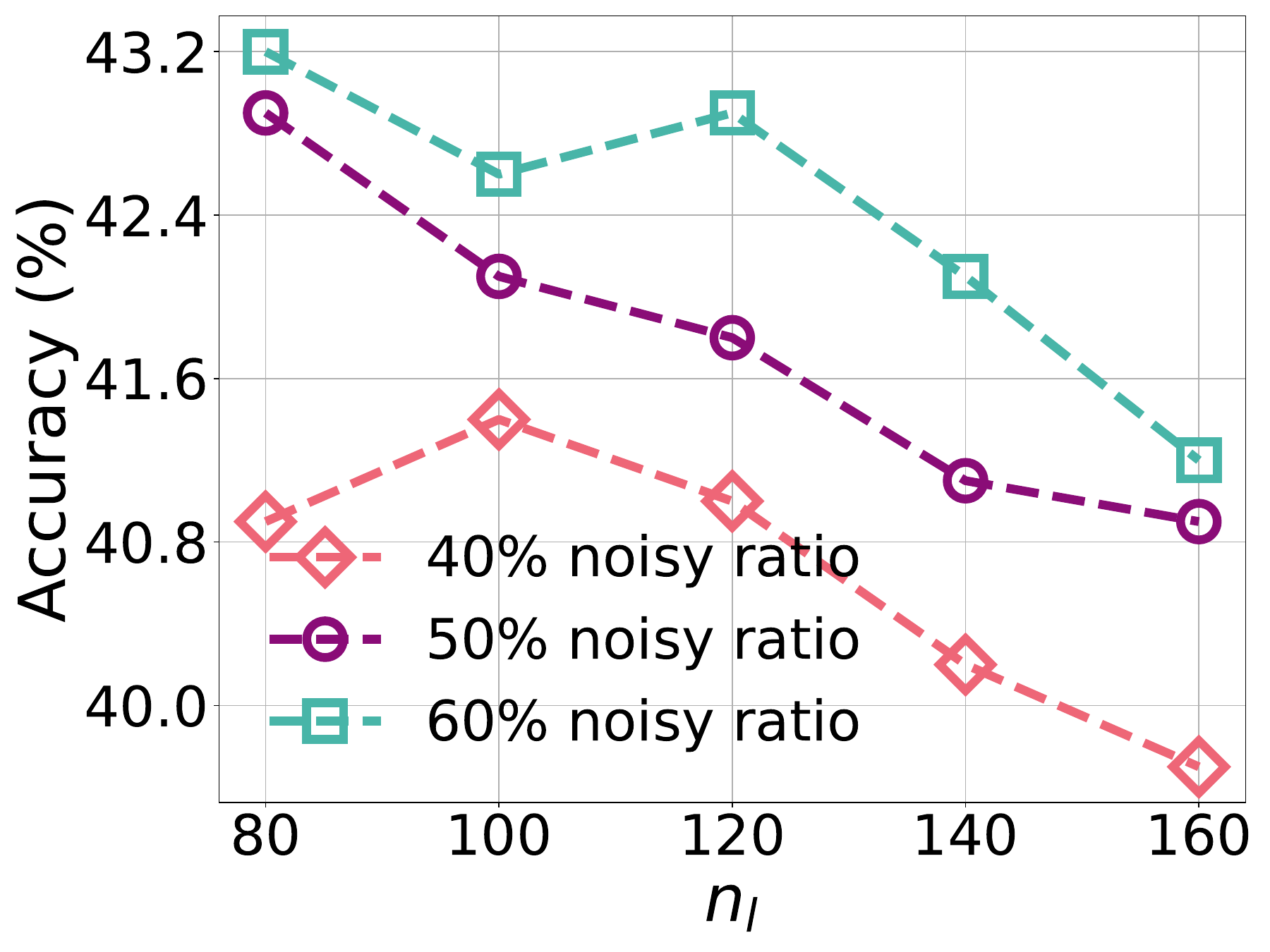}
  \label{fig:noisy_r-2}
  }
\hfil
  \subfloat[CIFAR-10]{
  \includegraphics[width=4cm,height=3cm]{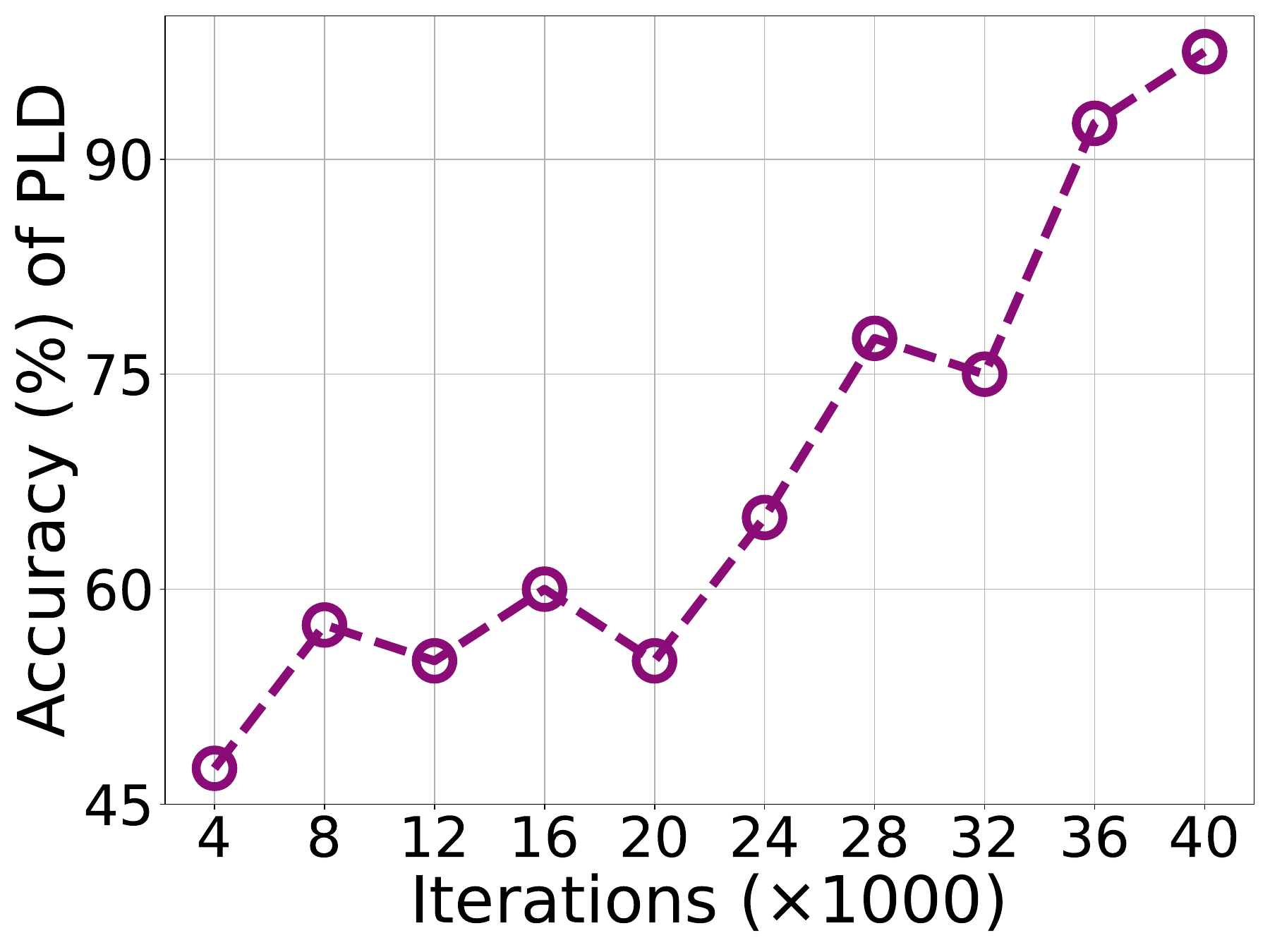}
  \label{fig:noisy_r-3}
  }
  \hfil
  \subfloat[CIFAR-100]{
  \includegraphics[width=4cm,height=3cm]{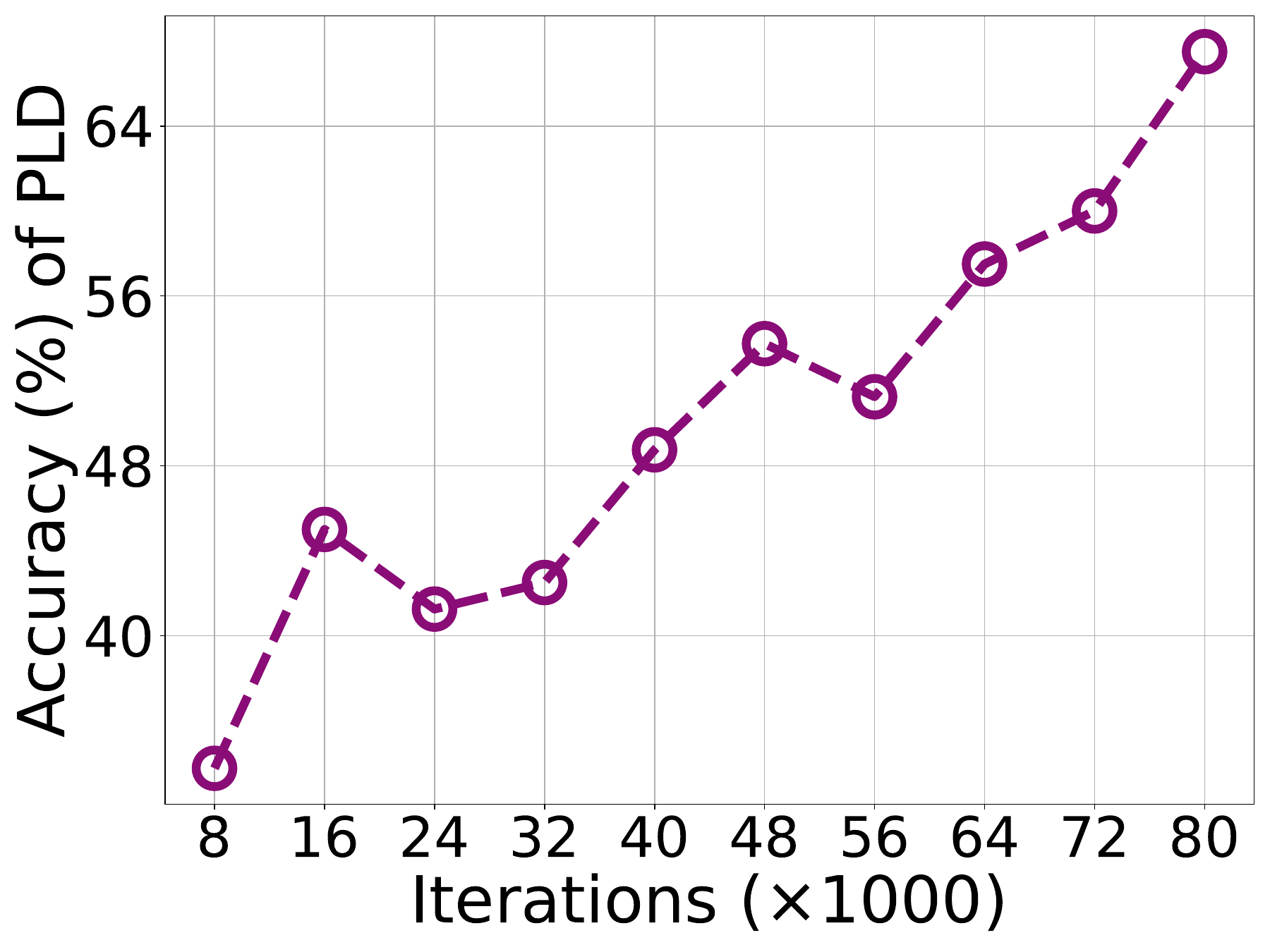}
  \label{fig:noisy_r-4}
  }
  \caption{(a) and (b): Ablation study under fixed noise ratios with varying pseudo-labeled data quantities. (c) and (d): Estimated noise rates of pseudo-labeled data (PLD) tracked across training iterations.} 
  \label{fig:noisy_r}
  \vskip 0in
\end{figure*}

\subsubsection{Comparisons with SSL Methods}
In addition, we provide further comparisons with SSL approaches in Tab. \ref{tab:con}. When using the same number of labels (\ie, $n_{l}$ in the context of ASD), as an unsupervised method, ASD achieves superior performance over the SSL methods used in RUC \cite{park2021improving} and SPICE \cite{niu2022spice} (\ie, MixMatch \cite{berthelot2019mixmatch} and FixMatch \cite{sohn2020fixmatch}) without any ground-truth labels. This indirectly confirms that our method successfully harnesses the potential of SSL learners to address deep clustering problems. Moreover, ASD obtains results comparable to the base SSL learners employed by ASD. Although there is still a performance gap, it's inevitable due to the noise label allocation in ASD's pseudo-label generation strategy. Namely, the baseline SSL learner's performance represents the possible room for the improvement of ASD's performance, and narrowing this gap with the baseline SSL learners will be our future goal.

\begin{table}[t!]
  \caption{Ablation study on $N_b$ conducted on CIFAR-10.
  }
  \label{tab:appab1}
  \vskip 0in
    \centering
    \footnotesize
     \begin{tabular}{@{}cccccc@{}}      
      \toprule
      $N_b$ & 100  & 200 & 500  & 1000 & 2000   \\ \cmidrule(){1-1}  \cmidrule(l){2-6}   
      Accuracy (\%)              & {90.2} & {91.1}& \underline{92.2} & \textbf{92.6} & 91.9\\ \bottomrule
      \end{tabular}
\end{table}

\begin{table}[t!]
  \caption{Ablation study on $N_t$ conducted on CIFAR-10.
  }
  \label{tab:appab2}
  \vskip 0in
    \centering
    \footnotesize
     \begin{tabular}{@{}cccccc@{}}      
      \toprule
      $N_t$ & 100  & 200 & 500  & 1000 & 2000   \\ \cmidrule(){1-1}  \cmidrule(l){2-6}   
      Accuracy (\%)              & \underline{92.4} & {91.5}& {91.8} & \textbf{92.6} & 90.8\\ \bottomrule
      \end{tabular}
\end{table}

In Tab. \ref{tab:con}, we compare with SSL methods using the same number of ground-truth labels as default $n_l=4k$ in ASD. In order to evaluate the performance of ASD more convincingly, we compare it with SSL methods using different numbers of ground-truth labels in Tab. \ref{tab:appssl}. It is worth noting that in general, the more ground-truth labels used by SSL methods, the better the performance. But this means higher requirements for manual annotation. ASD can achieve similar performance to them without any labels at all. In addition, when the number of labels used is very small (\eg, 10 labels), the SSL algorithms are even weaker than ASD. This shows that the fixed and extremely scarce supervision is even worse than the self-constructed dynamic supervision provided by ASD (although it contains noise), which further reflects the effectiveness of ASD.

\subsection{Ablation Studies} 
\label{sec:ab}
\subsubsection{Ablation on Pseudo-Labeled Data Sampling}
As shown in Figs. \ref{fig:exp-a2} and \ref{fig:exp-b2}, choosing a smaller $n_{l}$ appropriately would be more beneficial for ASD, which confirms our statement in Sec. \ref{sec:ppp}: although a larger $n_{l}$ can make us more confident in sampling all semantic classes for $D_{l}$ without knowing any prior, which allows the SSL learner to see the most comprehensive supervised signal, excessive $n_{l}$ inevitably introduces more noise and damages performance. Therefore, carefully weighing the size of $n_{l}$ will be more helpful in tapping into the potential of SSL learners. Meanwhile, re-sampling strategy in a epoch ensures that when $n_{l}$ is relatively small, $D_{l}$ can still cover the entire class space.

\begin{table*}[t!]
  \caption{
Comparison with recent advanced deep clustering paradigms that incorporate pretrained models and external supervision. We report the results of MutexMatch-based ASD$_{\text{PS}}$ using the same experimental settings as \cite{liu2024interactive} and  \cite{adaloglou2023exploring} to ensure fairness, and directly quote the reported results in them for comparison.}
  \label{tab:pre}
  \vskip 0in
    \centering
    \footnotesize
\begin{tabular}{@{}lccccccccccc@{}}
\toprule
  Datasets& \multirow{2}{*}{Backbone} & \multicolumn{3}{c}{CIFAR-10} & \multicolumn{3}{c}{CIFAR-100} & \multicolumn{3}{c}{STL-10} \\
\cmidrule(r){1-1} \cmidrule(lr){3-5} \cmidrule(lr){6-8} \cmidrule(lr){9-11} 
 Metrics & &ACC & NMI & ARI & ACC & NMI & ARI & ACC & NMI & ARI \\
\midrule
IDC \cite{liu2024interactive} w. TCL \cite{li2022twin} & ResNet-34& 92.7 & 84.4& 84.8&{69.4}& 58.1  &48.7&92.7& 85.3 & 84.6 \\
IDC \cite{liu2024interactive} w. ProPos \cite{huang2022learning} & ResNet-34 & 95.7 &  90.5& 90.9&\underline{78.3}& \underline{69.2} & \textbf{61.4}& - &- &-\\
IDC \cite{liu2024interactive} w. ASD (Ours) & ResNet-34 & 95.8 &  90.2& 90.1&\textbf{80.6}& \textbf{70.0} & \underline{58.5}& 94.1 &86.8 &85.2\\ \midrule
 TEMI  \cite{adaloglou2023exploring} w. DINO \cite{caron2021emerging}  & ViT-B/16  & 94.5 & 88.6 & 88.5 & 63.2 & 65.4 & {48.9} & \textbf{98.5} & \textbf{96.5} & \textbf{96.8} \\
 ASD (Ours) w. DINO \cite{caron2021emerging} & ViT-S/16 & \underline{96.7} & \underline{92.9} & \underline{92.4}& 62.2 & 63.9 & 47.7 & 93.4 & 81.5 & 72.6 \\
 ASD (Ours) w. DINO \cite{caron2021emerging} & ViT-B/16 & \textbf{96.9} & \textbf{93.0} & \textbf{92.5} & 65.3 & {66.6} & 48.8 & \underline{97.5}& \underline{89.6} & \underline{86.2}\\
\bottomrule
\end{tabular}
\end{table*}

\begin{table*}[t!]
  \caption{We evaluate MutexMatch-based ASD under different combinations of the PS strategy and pretraining, as well as under different K-Means initialization schemes. For pretraining, we subsequently follow up on \cite{van2020scan} by using the pretrained network weights officially provided by them for the corresponding three datasets.}
  \label{tab:ps}
  \vskip 0in
    \centering
    \footnotesize
    \setlength{\tabcolsep}{2.8mm}{
\begin{tabular}{@{}lccccccccccccccc@{}}
\toprule
Datasets                  &   \multicolumn{3}{c}{CIFAR-10}      &  \multicolumn{3}{c}{ CIFAR-100}   &    \multicolumn{3}{c}{ STL-10}   \\ \cmidrule(r){1-1} \cmidrule(lr){2-4} \cmidrule(lr){5-7} \cmidrule(lr){8-10} 
Metrics                   & ACC  & NMI      & ARI  & ACC  & NMI       & ARI  & ACC  & NMI    & ARI \\ \midrule
 ASD           & 92.6$\pm$3.0 &75.0$\pm$9.2     & 61.9$\pm$11.4 & 40.2$\pm$3.5 & 38.5$\pm$5.6      & 22.4$\pm$4.4 & 74.2$\pm$4.0 & 62.6$\pm$3.8     & 55.3$\pm$3.5  \\
  ASD w. pretraining           & \underline{93.2$\pm$3.1} &84.2$\pm$4.6     & 81.3$\pm$3.3 & 41.8$\pm$3.8 & 40.1$\pm$4.1      & 22.3$\pm$5.0 & \underline{76.7$\pm$5.2 }& \underline{64.0$\pm$6.4 }    & \underline{56.9$\pm$4.9}  \\ \midrule
  ASD$_{\text{PS}}$     &      93.1$\pm$1.5 & \underline{85.9$\pm$2.1 }&\underline{85.2$\pm$1.8}& \underline{43.6$\pm$1.9}& \underline{40.7$\pm$2.0}& \underline{23.7$\pm$2.1}& 75.9$\pm$3.6& 62.9$\pm$2.9 &56.4$\pm$2.8 \\
 ASD$_{\text{PS}}$  wo. updates          & 83.4$\pm$8.2 &61.0$\pm$12.6     & 52.3$\pm$10.5 & 32.7$\pm$8.5 & 28.6$\pm$8.6      & 18.4$\pm$4.9 & 70.5$\pm$5.9 & 57.3$\pm$4.8     & 51.0$\pm$3.2  \\
 ASD$_{\text{PS}}$ w. random     &       \textbf{93.1$\pm$1.9} & {84.1$\pm$2.9 }&{82.1$\pm$2.8}& {39.6$\pm$3.6}& {37.5$\pm$6.5}& {23.5$\pm$3.9}& {73.9$\pm$5.2}& {62.6$\pm$3.7} &{54.4$\pm$4.4} \\ 
 ASD$_{\text{PS}}$  w. pretraining         & \textbf{94.1$\pm$1.9} & \textbf{87.2$\pm$0.6} & \textbf{85.3$\pm$1.1} & \textbf{55.8$\pm$1.0} & \textbf{54.4$\pm$2.3} & \textbf{38.3$\pm$0.8} & \textbf{77.1$\pm$2.6} & \textbf{65.3$\pm$2.2} & \textbf{60.6$\pm$1.3}  \\ 
\bottomrule
\end{tabular}}
\end{table*}

\begin{figure}[t]
  \centering
  \subfloat[Random sampling]{
  \includegraphics[width=\linewidth]{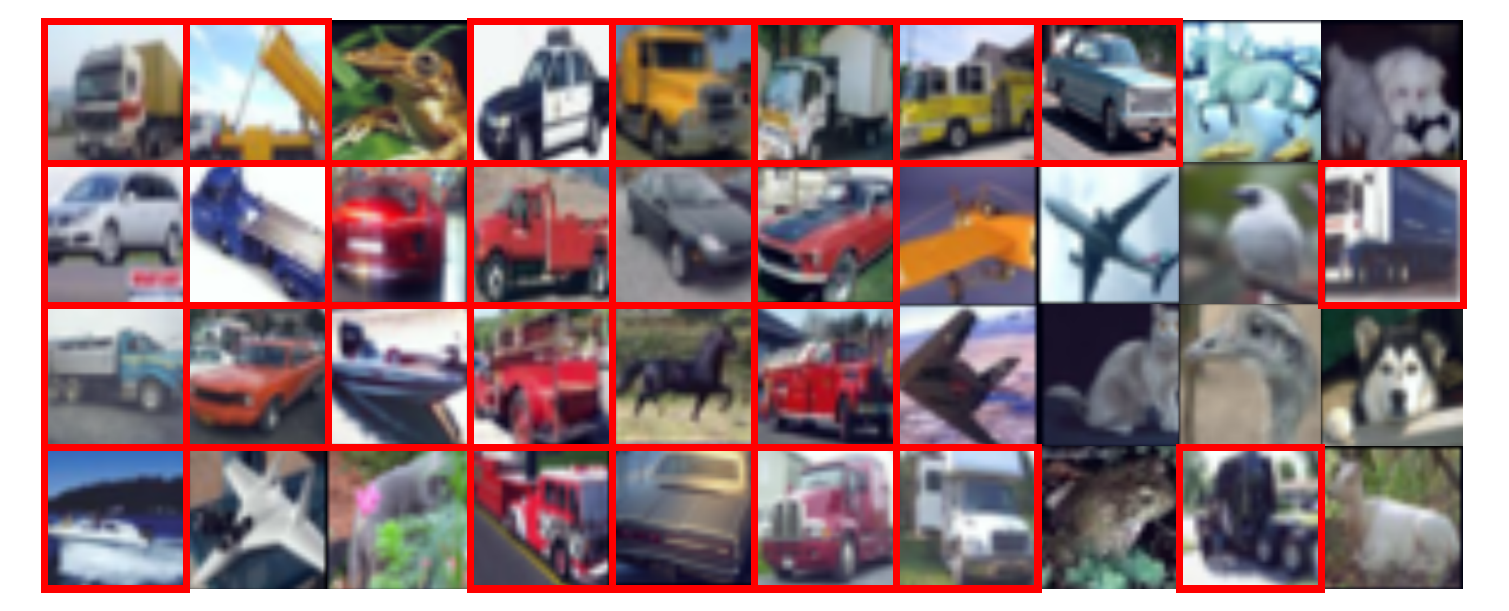}
  \label{fig:ps1}
  }
\hfil
  \subfloat[\textbf{P}rototypes accompanied by neighbors based \textbf{S}ampling
(PS)]{
  \includegraphics[width=\linewidth]{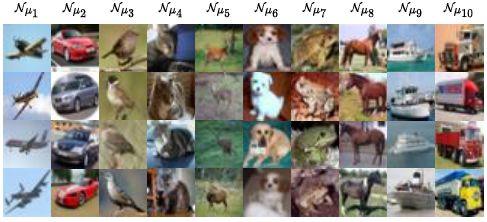}
  \label{fig:ps2}
  }
  \caption{Visualization of pseudo-labeled data sampled from CIFAR-10. (a): A failure case. We highlight in red boxes several visually similar trucks, automobile, and ships. Unfortunately, in this particular round of random sampling, these categories are overrepresented, failing to achieve a balanced semantic coverage and lacking sufficient representativeness. (b):  Each column corresponds to the set of nearest neighbors $\mathcal{N}{\mu^{(i)}}$ associated with a distinct prototype $\mu^{(i)}$ identified via K-Means clustering in the feature space. These neighbors are selected as pseudo-labeled samples to initiate training. We can intuitively observe that these samples tangibly reflect the semantics of the corresponding clusters.}
  \label{fig:ps}
\end{figure}
Next, in order to explore the performance offline of ASD, we need to face the \textbf{worst-case} scenario, which is ASD unfortunately fails to sample all semantic classes for $D_{l}$ in each iteration. We deliberately control to randomly drop $N_{miss}$ classes in each iteration. As shown in Figs. \ref{fig:exp-c2} and \ref{fig:exp-d2}, the performance of ASD still exhibits objective robustness under this setting. Note that in actual situations, there are almost no missing classes. For example, if $n_{l}=40$ is used for CIFAR-10, as shown in Fig \ref{fig:sa}, the probability of containing all classes is about 90\%. Although the SSL learner may not encounter all semantic classes in an iteration, it can still benefit from the class transition matrix containing information with current unseen classes from previous iterations.

\subsubsection{Ablation on Cluster-Level Label Assignment}
 To investigate effectiveness of CTT-based $\phi_{l}$ in ASD, we replaced it with $\phi_{l}(x^{(i)}_{i})=\texttt{KM}_{k}(\{F(x^{(1)}_{l},\cdots,F(x^{(n_{l})}_{l})\})_{i}$, where $\texttt{KM}_{k}(\cdot)$ is the K-Means algorithm clustering samples into $k$ classes based on their features.  As mentioned in Sec. \ref{sec:ctt}, due to the difficulty of cluster matching with sample-level clustering, we simply abandon the resampling strategy and only sample once. As shown in Figs. \ref{fig:exp-a} and \ref{fig:exp-b}, CTT-based $\phi_{l}$ consistently maintains higher clustering accuracy on $x^{(i)}_{l}$, ensuring that $x^{(i)}_{l}$ provides accurate supervision to SSL learners. Additionally, Figs. \ref{fig:exp-c} and \ref{fig:exp-d} also demonstrate that using $\mathbf{C}'$ to represent the similarity between instance-level classes is more discriminative than using the features of pseudo-labeled data in current iteration, as $\mathbf{C}'$ contains historical information and knowledge learned from all unlabeled data. 

Then, we further examine the robustness of ASD under varying noise levels of cluster-level labels obtained by CTT-based $\phi_{l}$ for pseudo-labeled data. We design a controlled simulation where the noise ratio is fixed artificially, while $n_l$ is varied. The results, presented in Figs. \ref{fig:noisy_r-1} and \ref{fig:noisy_r-2}, show that even when the noise ratio remains constant, increasing the absolute number of labeled data tangibly degrades performance. This finding further confirms our discussion in Sec. \ref{sec:plds} regarding the trade-off between pseudo-label quantity and noise sensitivity. It justifies our design choice of using a relatively small value for $n_l$, which not only facilitates semantic coverage and high-quality supervision, but also implicitly controls the magnitude of potential noise introduced per iteration. Moreover, it is important to highlight that in the actual operation of ASD, the noise level is not static—it evolves dynamically across training. As the learned feature representations gradually improve, the quality of clustering and, consequently, the pseudo-labels becomes progressively more reliable.  As shown in Figs. \ref{fig:noisy_r-3} and \ref{fig:noisy_r-4}, the noise rate exhibits a clear downward trend during training, confirming the presence of a self-correcting mechanism in ASD. Through iterative representation refinement and re-sampling, the framework naturally reduces the impact of early-stage noisy assignments and converges toward increasingly accurate supervision. Overall, these results underscore ASD’s robustness under noisy conditions.

\subsubsection{Ablation on Hyper-Parameters} We further conduct ablation experiments on hyper-parameters of ASD. As shown in Tabs. \ref{tab:appab1} and \ref{tab:appab2} (ASD is implemented on MutexMatch \cite{duan2022mutexmatch}), ASD exhibits insensitivity to $N_b$ (the number of tracker batches) and $N_t$ (the update frequency of $\phi_l$). For $N_b$, if $N_b$ is too large, the update intensity of class transition matrix (\ie, $\mathbf{C}$) will be too small, while if $N_b$ is too small, the update intensity will be too large, both of which are not satisfactory. For $N_t$, a larger $N_t$ means that the new knowledge captured by $\mathbf{C}$ will be transmitted to $\phi_l$ more slowly, which is not conducive to label mapping for pseudo-labeled data. Thus, we need choose $N_b$ and $N_t$ with moderate sizes for ASD.

\subsection{Discussions on Pretraining and PS}
\label{sec:dcp}
\subsubsection{Pretraining}
Although ASD is originally designed as an out-of-the-box framework that does not rely on pretraining (unlike works such as \cite{van2020scan, niu2022spice} discussed in Sec.\ref{sec:intro}), it remains fully compatible with advanced pretraining strategies and modern vision backbones. To assess this compatibility, we further evaluate ASD using a stronger feature extractor—Vision Transformer (ViT) \cite{dosovitskiy2020image}—in combination with the state-of-the-art self-supervised pretraining method DINO \cite{caron2021emerging}, following recent best practices \cite{adaloglou2023exploring}. As reported in Tab.\ref{tab:pre}, both the enhanced backbone and pretraining significantly boost clustering performance, demonstrating that ASD can effectively benefit from high-quality feature initialization. Meanwhile, we also observe a recent trend in deep clustering research that goes beyond traditional from-scratch training pipelines. Several emerging approaches shift toward leveraging pretrained models and even minimal external supervision to extract richer semantics. For instance, IDC \cite{liu2024interactive} introduces interactive supervision based on high-value sample selection—incorporating hardness, representativeness, and diversity—to improve clustering decisions with minimal human input. While such strategies offer impressive results, they typically require external signals. 

To ensure a fair comparison, we include recent advanced methods such as TEMI \cite{adaloglou2023exploring} and IDC \cite{liu2024interactive} in Tab.~\ref{tab:pre}. Importantly, although ASD uses a much smaller batch size (32) compared to TEMI (512), it still achieves comparable or even superior performance, particularly on CIFAR-10 and CIFAR-100. Furthermore, we also adopt the clustering network trained by ASD within the IDC framework and observe that it leads to competitive or even improved results, indicating that ASD produces high-quality clustering models that can serve as a strong initialization or component for more elaborate pipelines like IDC. These results affirm that ASD not only performs well under conventional training settings but also scales robustly with modern architectures and pretraining techniques.

\subsubsection{PS}
To better isolate and understand the individual and combined effects of self-supervised pretraining and the PS strategy, we conduct controlled ablation experiments under three representative settings: (1) applying PS without pretraining (ASD$_{\text{PS}}$), (2) applying pretraining without PS (ASD w. SimCLR), and (3) applying both together (ASD$_{\text{PS}}$ w. SimCLR). For all cases involving pretraining, we follow standard practice and use pretrained weights officially provided by \cite{van2020scan} for the corresponding datasets. The results are summarized in Tab.~\ref{tab:ps}, which show performance across three datasets.

Beyond isolating the main factors, we further investigate the robustness of the proposed PS strategy, with a particular focus on two aspects: centroid initialization and iterative sampling for pseudo-labeled data $\mathcal{D}_l$. Regarding the former, we compare the standard \texttt{k-means++} initialization (used in our main experiments) against random initialization (ASD$_{\text{PS}}$ w. random). The results demonstrate that PS maintains broadly consistent performance across both initialization schemes, suggesting limited sensitivity to the choice of seed. As for iterative sampling, we evaluate a variant (ASD$_{\text{PS}}$ wo. updates) in which clustering is performed only once in the first iteration, and the resulting pseudo-labeled samples are fixed for the remainder of training. Compared to our default setting that updates $\mathcal{D}_l$ periodically, this fixed $\mathcal{D}_l$ baseline yields significantly lower accuracy and higher variance. This highlights the necessity of iterative re-sampling: on one hand, as feature representations become increasingly discriminative, PS is able to generate higher-quality pseudo-labels for $\mathcal{D}_l$; on the other hand, periodic updates mitigate the risk of the model being constrained by a suboptimal or biased subset of pseudo-labeled data. By refreshing the training supervision at each iteration, the model can progressively access a broader and more representative sample space. The results in Tab.~\ref{tab:ps} also show the advantages of our iterative re-sampling strategy.

Meanwhile, both PS and pretraining consistently contribute to reducing performance variance. Vanilla ASD shows performance fluctuations, partly due to random sampling. Though simple, it may select unrepresentative or poor-quality samples early on, leading to weak supervision and instability. As shown in Fig. \ref{fig:ps1}, a failure case occurs when overrepresented classes dominate the sampled data, causing ASD to collapse by clustering all samples into the same group. PS deliberately samples pseudo-labeled data centered around diverse and representative prototypes (see Fig. \ref{fig:ps2}), promoting broad semantic coverage and more reliable supervision for robust SSL training. But we have to mention that that despite this, some limitations remain, especially in fine-grained or imbalanced datasets where ensuring full semantic coverage per iteration is challenging without ground-truth labels. In such cases, PS may still focus on a subset of visually similar classes, leading to suboptimal sampling and biased pseudo-label distributions. We provide detailed analysis in the Appendix \ref{app:adl}.

\begin{table}[t!]
  \caption{
Clustering performance comparisons on SVHN \cite{netzer2011reading} in the setting of semi-supervised clustering. To be fair, we directly quote the results reported by \cite{ding2024graph}.}
  \label{tab:sv}
  \vskip 0in
    \centering
    \footnotesize

\begin{tabular}{cccc}
\toprule Method & ACC& NMI& ARI\\
\midrule DEC  \cite{xie2016unsupervised} & 13.7 & 11.1 & 10.6 \\
DAE \cite{vincent2010stacked} & 11.0 & 9.8 & 8.2 \\
VAE \cite{kingma2013auto} & 10.9 & 9.8 & 8.6 \\
DeCNN \cite{zeiler2010deconvolutional} & 9.3 & 9.1 & 7.3 \\
GAN \cite{radford2015unsupervised} & 11.2 & 9.6 & 9.0 \\
JULE \cite{yang2016joint} & 15.2 & 11.1 & 11.4 \\
DAC \cite{chang2017deep} & 16.5 & 11.3 & 13.8 \\
DCCM \cite{wu2019deep} & 14.9 & 11.2 & 11.4 \\
DAFC \cite{tan2023deep} & 17.0 & 11.5 & 14.3 \\
MFCVAE \cite{falck2021multi} & 56.5 & - & - \\
GSDC \cite{ding2024graph}  & 66.4 & 42.5 & 42.1 \\
\rowcolor{GrayMedium} ASD (Ours) & \textbf{72.6} &\textbf{56.2} &\textbf{50.4} \\
\bottomrule
\end{tabular}

\end{table}

\subsection{More Clustering Task Settings}
We further evaluate ASD under the semi-supervised clustering setting, where a small number of labeled samples are available alongside a large amount of unlabeled data. We follow the experimental protocol of GSDC \cite{ding2024graph} and conduct experiments on SVHN \cite{netzer2011reading}.  Note that the labeled data here refers to samples with ground-truth labels, rather than pairwise constraints or other forms of weak supervision. These labels can be used during the supervised training phase but are not accessible during the unsupervised learning process. In the supervised training phase, we directly activate the SSL learner following its original training protocol—meaning we no longer construct pseudo-labeled data. The rest of the ASD pipeline remains unchanged. In the subsequent unsupervised phase, we resume the pseudo-label sampling process. For fair comparison, we adopt the same backbone and directly quote the reported results of existing baselines. As shown in Tab. \ref{tab:sv}, ASD significantly outperforms all prior methods, including recent state-of-the-art approaches. This demonstrates the strong adaptability of ASD to semi-supervised clustering scenarios, even without any modification. Given ASD’s intrinsic alignment with semi-supervised learning frameworks, it can naturally excel in the semi-supervised.

\begin{table}[t]
\centering
\caption{Time complexity and runtime of ASD's main components ---\textit{``PLDS''}: Pseudo-Labeled Data Sampling; \textit{``OT-SA''}: OT-Based Semantic Alignment; \textit{``CTT-LM''}: CTT-Based Label Mapping ($I$ denotes the number of iterations for $k$-medoids).  We report the runtime of one iteration on CIFAR-10 with MutexMatch-based ASD, GeForce RTX 3090 GPU and Intel Xeon Gold 6226R CPU @2.90GHz. ``\textit{Ratio}'' represents the percentage of the runtime relative to the total. }
\footnotesize
 \setlength{\tabcolsep}{3mm}{
\begin{tabular}{@{}c|c|c|c@{}}
\toprule
Component     & Complexity & Runtime & Ratio  \\ \midrule
 PLDS &   $\mathcal{O}(n_l)$   &   0.36ms & 0.15\% \\ 
OT-SA &   $\mathcal{O}(n_l^2) $ &        1.95ms & 0.80\%\\ 
CTT-LM &     $\mathcal{O}(I \cdot k(n_{l}-k)^2) + \mathcal{O}(B)$&   2.68ms & 1.10\% \\ 
Total training &  -    &  243.72ms & -  \\
\bottomrule
\end{tabular}
}
\label{tab:time}
\end{table}

\subsection{Efficiency Analysis}
\label{sec:time}
Denoting the number of sampled pseudo-labeled data as $n_{l}$, the number of clusters as $k$ and the batch size as $B$, we provide the efficiency analysis of ASD in Tab. \ref{tab:time}. For time complexity, PLDS depends on random sampling; CTT-LM mainly depends on $k$-medoids algorithm \cite{kaufman2009finding} conducted on CTT matrix where CTT is a independent procedure with a loop at the time complexity of $\mathcal{O}(B)$ (\ie, it is performed once for each sample in the batch). For OT-SA, we use POT \cite{flamary2021pot} to solve OT with Sinkhorn solver. Thus, its time complexity is nearly $\mathcal{O}(n_l^2)$ \cite{flamary2021pot}, where $n_l$ is the amount of pseudo-labeled data with a relatively small value. As shown in Tab. \ref{tab:time}, {the runtime of OT-SA accounts for only 0.80\% of the total}. Moreover, the complexities of ASD-specific components are mainly depend on a relatively small $n_l$, {contributing only about 2.05\% runtime}. The main time cost of ASD is attributed to the SSL learner, or pretraining in case of DC-embedded ASD.

\section{Conclusion}
In this work, we propose ASD, an adaptor for SSL that enables out-of-the-box clustering. First, we perform random sampling to obtain pseudo-labeled data. By setting up an instance-level classifier trained on the pseudo-labeled data with labels aligned semantically, the model can perform instance-level classification on all unlabeled data. 
Then, with the similarity matrix obtained by tracking the class transitions of instance-level predictions on the unlabeled data, we assign cluster-level labels to pseudo-labeled data. Finally, we leverage the pseudo-labeled data with assigned labels to activate a general SSL learner to learn from unlabeled data for clustering. We believe that our work could continue to evolve with the advancement of SSL and further contribute to deep clustering.

{\footnotesize
\bibliographystyle{IEEEtran}
\bibliography{egbib}
}

\appendices

\begin{figure*}[t]
  \centering
  \resizebox{\linewidth}{!}{   \includegraphics[]{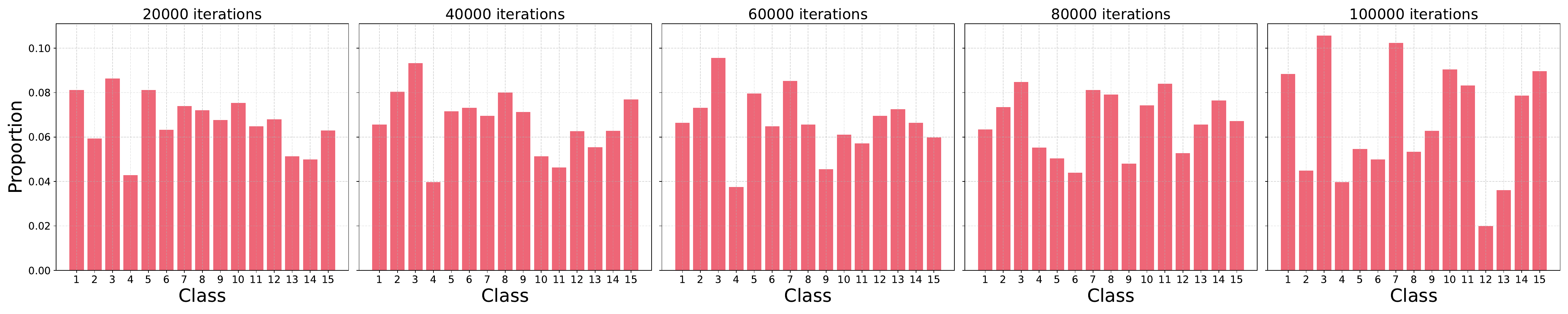}}
  \caption{Visualization of a failure case showing the class distribution of pseudo-labels assigned to unlabeled data during training on Image-Dogs. The distribution remains imbalanced, indicating that the model is consistently influenced by skewed pseudo-labeled supervision.} 
  \label{fig:ps-distri}
\end{figure*}

\section{Proof of Theorem \ref{the}}
\label{app:proof}
\addtocounter{theorem}{-1}
\begin{theorem}
Given $n$ samples with $k$ classes and a uniform class-distribution (\ie, the number of samples per class is $\frac{n}{k}$), the probability of randomly selecting $n_{l}$ samples ($n_{l} \geq k$) containing all $k$ classes $P_{all}(n_{l}, k, n)$ is given by:
\begin{equation}
P_{all}(n_{l}, k, n) = 1 - \sum_{i=1}^{k} (-1)^{(i-1)} \frac{\binom{k}{i} \binom{n - i(\frac{n}{k})}{n_{l}}}{\binom{n}{n_{l}}},
\end{equation}
where $\binom{a}{b}$ represents the number of combinations.
\end{theorem}
\begin{proof}
Since the samples of all classes are uniformly distributed, we can obtain each class has $\frac{n}{k}$ samples. We will derive the probability by first calculating the probability of the complementary event—that at least one class is missing.

Let's define a term $P_{term}(n_l,k,n,i)$ representing the $i$-th sum in the Inclusion-Exclusion series. This term corresponds to the sum of probabilities for every possible intersection of $i$ ``missing class'' events. It is given by:
\begin{equation}
P_{term}(n_l,k,n,i) = \frac{ \binom{k}{i} \binom{n - i(\frac{n}{k})}{n_l}}{\binom{n}{n_l}}.
\label{eq:miss}
\end{equation}

By the Principle of Inclusion-Exclusion, the exact probability of having at least one missing class is the alternating sum of the terms defined in Eq. (\ref{eq:miss}):
\begin{align}
P(\text{at least 1 missing class}) &= \sum_{i=1}^{k} (-1)^{(i-1)} P_{term}(n_{l},k,n,i) \nonumber\\
&= \frac{\sum_{i=1}^{k} (-1)^{(i-1)} \binom{k}{i} \binom{n - i(\frac{n}{k})}{n_{l}}}{\binom{n}{n_{l}}}.
\label{eq:7}
\end{align}
 Then, we use the complementary principle, subtracting the probability of having at least one missing class from $1$ to obtain the exact probability of containing all $k$ classes:
\begin{align}
P_{all}(n_{l}, k, n) &= 1 - P(\text{at least 1 missing class}) \nonumber \\
&= 1 - \sum_{i=1}^{k} (-1)^{(i-1)} \frac{\binom{k}{i} \binom{n - i(\frac{n}{k})}{n_{l}}}{\binom{n}{n_{l}}}.
\label{eq:8}
\end{align}
\end{proof}
\begin{figure*}[t]
  \centering
  \resizebox{0.73\linewidth}{!}{   \includegraphics[]{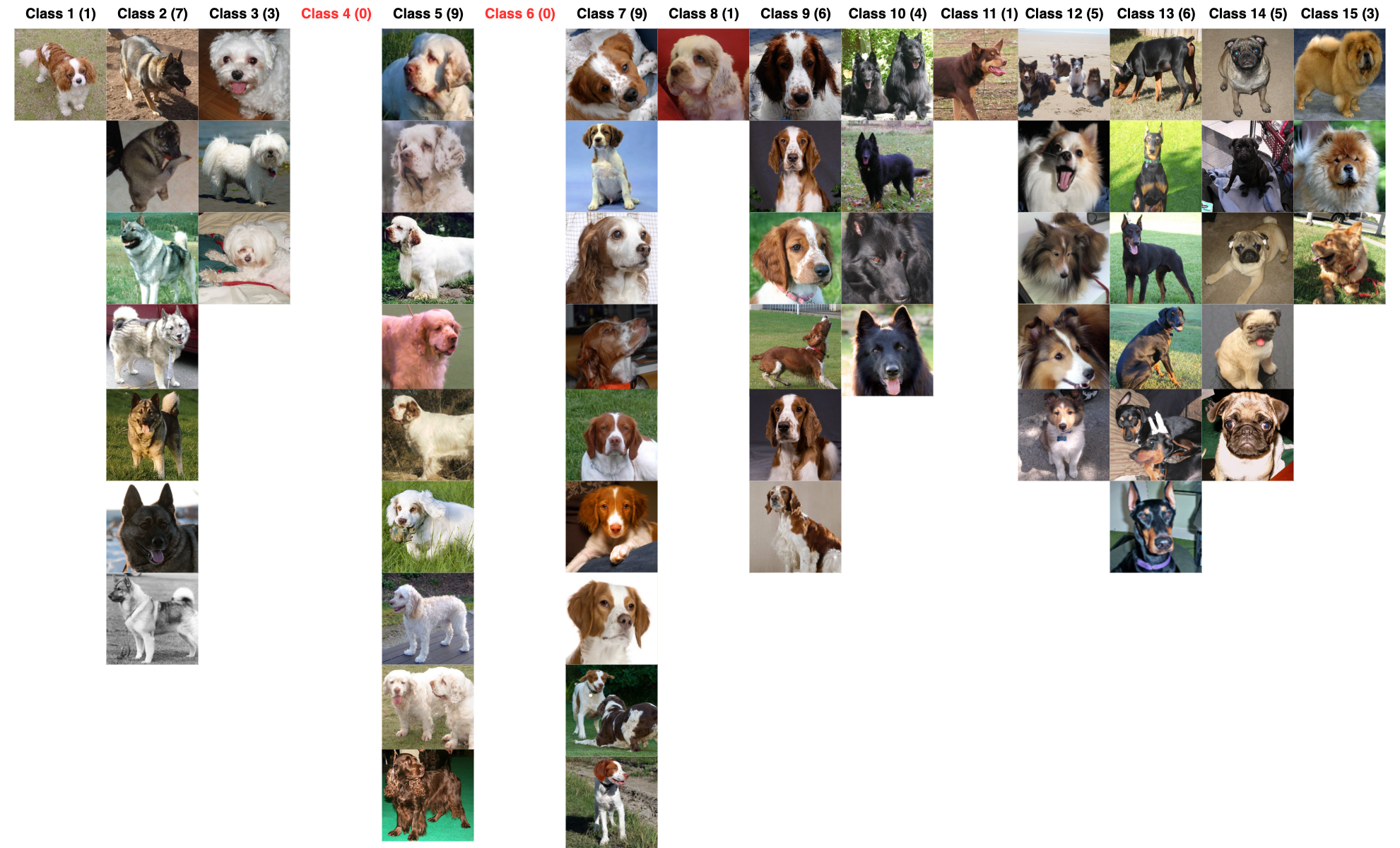}}
  \caption{Visualization of a failure case from ImageNet-Dogs using PS. Each column shows pseudo-labeled samples belonging to each class obtained by PS. While PS improves sample quality overall, some semantic classes (e.g., class 4, 6) remain uncovered after  20,000 training iterations.}
  \label{fig:ps-f}
\end{figure*}
\section{Additional Discussions on Limitations}
\label{app:adl}
While the PS strategy greatly improves training stability and semantic coverage in ASD, some limitations remain, especially on fine-grained or imbalanced datasets. Without ground-truth labels, it is difficult to guarantee that every iteration’s pseudo-labeled set covers all semantic categories. This issue is more pronounced in datasets like ImageNet-Dogs, where classes are visually similar and intra-class variance is small. In such cases, prototype discovery based on shallow or randomly initialized features may produce cluster centroids that represent only a narrow semantic range, causing pseudo-labeled samples to concentrate on a few visually similar classes even after many training iterations (Fig.~\ref{fig:ps-f}). This leads to biased pseudo-label distributions (Fig.~\ref{fig:ps-distri}), where the model overfits overrepresented classes and neglects others, hindering convergence to semantically faithful clusters. Addressing these challenges may require more advanced techniques such as long-tailed-aware reweighting, semantic coverage tracking, and dynamic sampling adjustments based on class imbalance or semantic drift indicators, which will be explored in future work.

\vspace{-2em}
\begin{IEEEbiography}[{\includegraphics[width=1in,height=1.25in,clip,keepaspectratio]{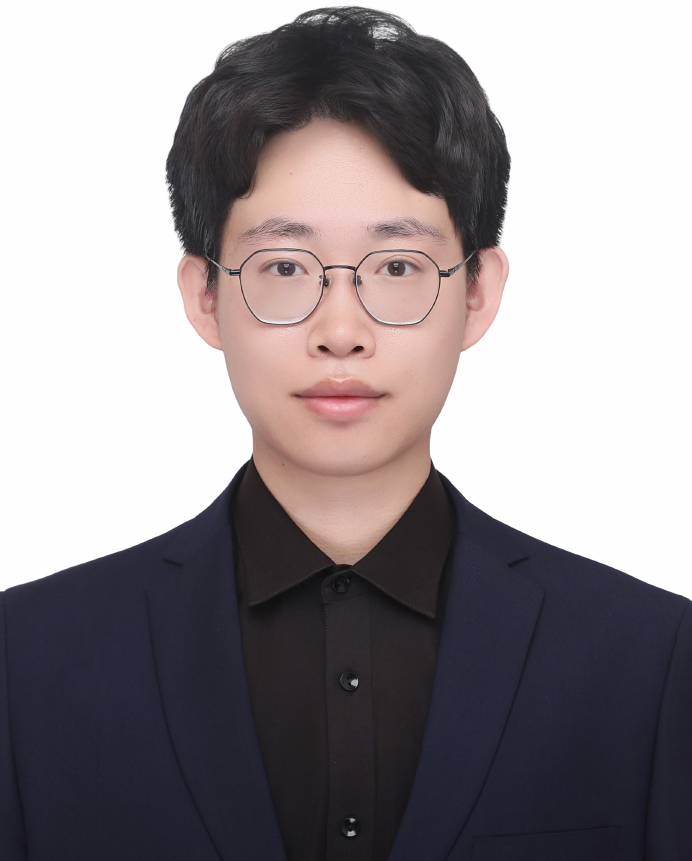}}]{Yue Duan} received the B.Eng. degree from the School of Computer Science and Technology, Harbin Institute of Technology, Weihai, China, in 2017. He is currently pursuing the Ph.D. degree at the School of Computer Science, Nanjing University, China. His research interests include semi-supervised learning, multimodal learning, and large multimodal models.
\end{IEEEbiography}
\begin{IEEEbiography}[{\includegraphics[width=1in,height=1.25in,clip,keepaspectratio]{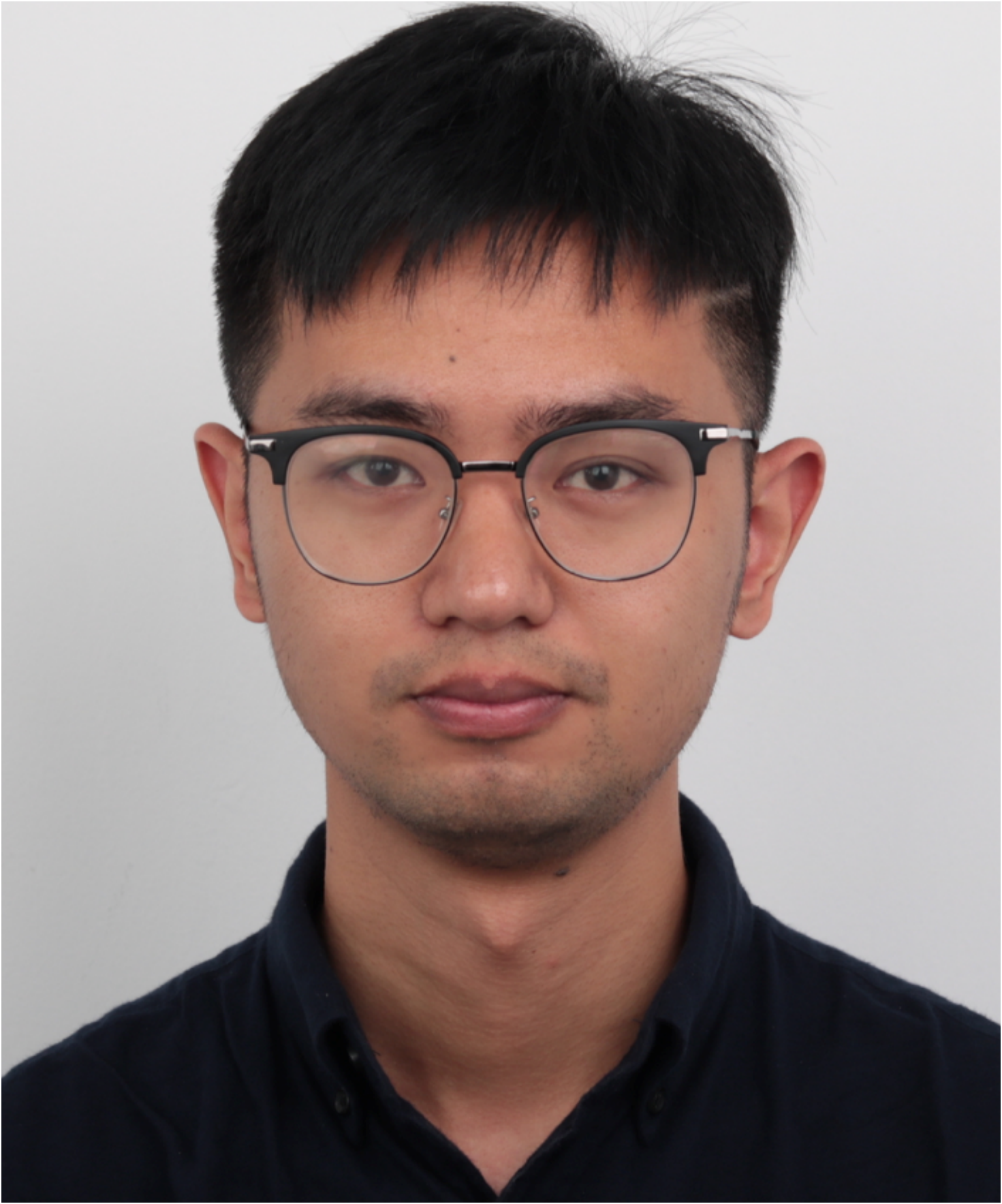}}]{Lei Qi} is currently an Associate Professor with the School of Computer Science and Engineering, Southeast University, China. His current research interests include some ML methods, such as domain adaptation, semi-supervised learning, unsupervised learning, and meta-learning. For applications, he mainly focuses on person re-identification and image segmentation.
\end{IEEEbiography}
\begin{IEEEbiography}[{\includegraphics[width=1in,height=1.25in,clip,keepaspectratio]{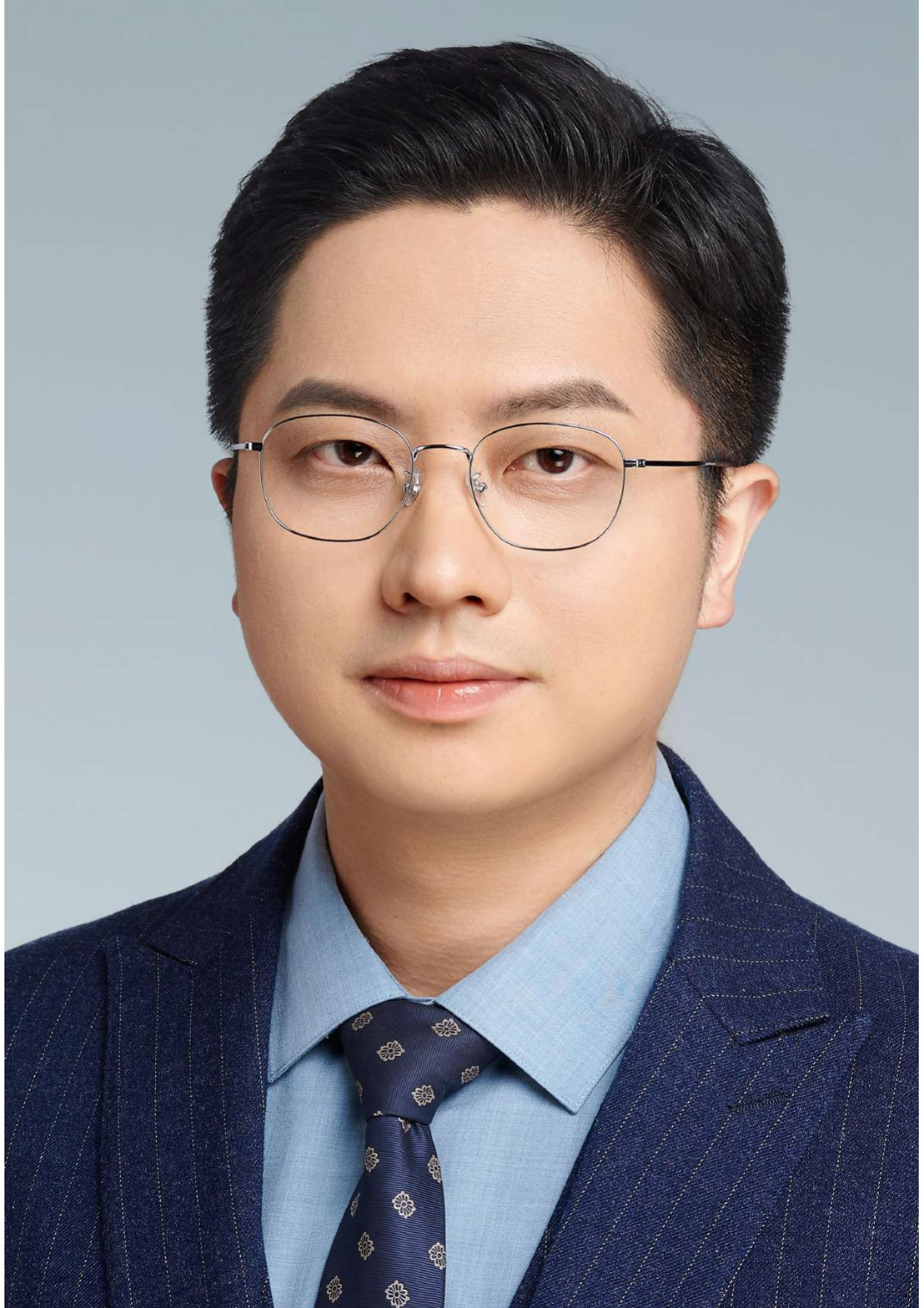}}]{Yinghuan Shi} is currently a Professor at the School of Computer Science, Nanjing University, and he is also affiliated with National Institute of Healthcare Data Science, Nanjing University. He received the B.Sc. and Ph.D. degrees both from Computer Science, Nanjing University, in 2007 and 2013, respectively. His research interests include machine learning, pattern recognition, and medical image analysis. He has published more than 80 papers in CCF-A Conference and IEEE Transactions.
\end{IEEEbiography}
\begin{IEEEbiography}[{\includegraphics[width=0.95in,height=1.25in,clip,keepaspectratio]{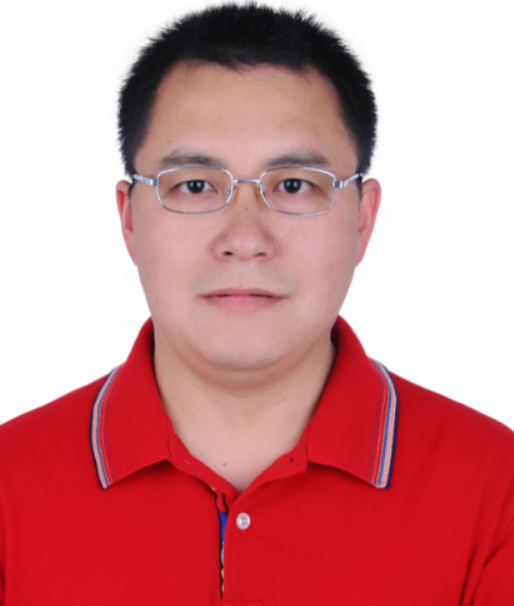}}]{Yang Gao} (Senior Member, IEEE) is a Professor in the School of Intelligence Science and Technology, Nanjing University. He is currently directing the Reasoning and Learning Research Group in Nanjing University. He has published more than 120 papers in top-tired conferences and journals. He also serves as Program Chair and Area Chair for many international conferences. His current research interests include artificial intelligence and machine learning.
\end{IEEEbiography}
\end{document}